\documentclass[11pt,twoside]{article}
\usepackage[margin=1in]{geometry}

\usepackage[utf8]{inputenc} % allow utf-8 input
\usepackage[T1]{fontenc}    % use 8-bit T1 fonts
\usepackage{hyperref}       % hyperlinks
\usepackage{url}            % simple URL typesetting
\usepackage{booktabs}       % professional-quality tables
\usepackage{nicefrac}       % compact symbols for 1/2, etc.
\usepackage{microtype}      % microtypography
\usepackage{xcolor}         % colors
\usepackage{graphicx}
\usepackage{mathrsfs}
\usepackage{amsfonts, amsmath,dsfont,amssymb,amsthm,stmaryrd,bbm}
\usepackage{mathtools}
\usepackage{caption}
\usepackage{subcaption}
\usepackage{enumitem}
\usepackage{wrapfig}
\usepackage{algorithm}
\usepackage{algorithmic}
\newtheorem{conj}{Conjecture}
\newtheorem{theorem}{Theorem}
\newtheorem*{theorem*}{Theorem}

\newtheorem{remark}[conj]{Remark}

\newtheorem{defn}[conj]{Definition}
\newtheorem{coro}{Corollary}

\newcommand{\f}[1]{\mathbf{#1}}
\newcommand{\bb}[1]{\mathbb{#1}}
\newcommand{\ca}[1]{\mathcal{#1}}

\newcommand{\s}[1]{\mathsf{#1}}

\newcommand{\Su}{\mathsf{occ}}
\newcommand{\Tr}{g}
\newcommand{\MLC}{\textsc{MLC}}
\newcommand{\MLR}{\textsc{MLR}}

\newcommand{\remove}[1]{}

\newcommand\nc\newcommand
\renewcommand{\cal}[1]{\mathcal{#1}}

\DeclarePairedDelimiter{\br}{\lparen}{\rparen}

\nc\bfa{{\bf{a}}}\nc\bfA{{\boldsymbol A}}\nc\cA{{\cal A}} \nc\fA[1]{A\br*{#1}} \nc\fa[1]{a\br*{#1}}  \nc\rmA{\mathrm{A}} \nc\rma{\mathrm{a}}
\nc\bfb{{\bf{b}}}\nc\bfB{{\boldsymbol B}}\nc\cB{{\cal B}} \nc\fB[1]{B\br*{#1}} \nc\fb[1]{b\br*{#1}}  \nc\rmB{\mathrm{B}} \nc\rmb{\mathrm{b}}
\nc\bfc{{\bf{c}}}\nc\bfC{{\boldsymbol C}}\nc\cC{{\cal C}} \nc\fC[1]{C\br*{#1}} \nc\fc[1]{c\br*{#1}}  \nc\rmC{\mathrm{C}} \nc\rmc{\mathrm{c}}
\nc\bfd{{\bf{d}}}\nc\bfD{{\boldsymbol D}}\nc\cD{{\cal D}} \nc\fD[1]{D\br*{#1}} \nc\fd[1]{d\br*{#1}}  \nc\rmD{\mathrm{D}} \nc\rmd{\mathrm{d}}
\nc\bfe{{\bf{e}}}\nc\bfE{{\boldsymbol E}}\nc\cE{{\cal E}} \nc\fE[1]{E\br*{#1}} \nc\fe[1]{e\br*{#1}}  \nc\rmE{\mathrm{E}} \nc\rme{\mathrm{e}}
\nc\bff{{\bf{f}}}\nc\bfF{{\boldsymbol F}}\nc\cF{{\cal F}} \nc\fF[1]{F\br*{#1}} \nc\ff[1]{f\br*{#1}}  \nc\rmF{\mathrm{F}} \nc\rmf{\mathrm{f}}
\nc\bfg{{\bf{g}}}\nc\bfG{{\boldsymbol G}}\nc\cG{{\cal G}} \nc\fG[1]{G\br*{#1}} \nc\fg[1]{g\br*{#1}}  \nc\rmG{\mathrm{G}} \nc\rmg{\mathrm{g}}
\nc\bfh{{\bf{h}}}\nc\bfH{{\boldsymbol H}}\nc\cH{{\cal H}} \nc\fH[1]{H\br*{#1}} \nc\fh[1]{h\br*{#1}}  \nc\rmH{\mathrm{H}} \nc\rmh{\mathrm{h}}
\nc\bfi{{\bf{i}}}\nc\bfI{{\boldsymbol I}}\nc\cI{{\cal I}} \nc\fI[1]{I\br*{#1}} \nc\rmI{\mathrm{I}} \nc\rmi{\mathrm{i}}
\nc\bfj{{\bf{j}}}\nc\bfJ{{\boldsymbol J}}\nc\cJ{{\cal J}} \nc\fJ[1]{J\br*{#1}} \nc\fj[1]{j\br*{#1}} \nc\rmJ{\mathrm{J}} \nc\rmj{\mathrm{j}}
\nc\bfk{{\bf{k}}}\nc\bfK{{\boldsymbol K}}\nc\cK{{\cal K}} \nc\fK[1]{K\br*{#1}} \nc\fk[1]{k\br*{#1}} \nc\rmK{\mathrm{K}} \nc\rmk{\mathrm{k}}
\nc\bfl{{\bf{l}}}\nc\bfL{{\boldsymbol L}}\nc\cL{{\cal L}} \nc\fL[1]{L\br*{#1}} \nc\fl[1]{l\br*{#1}} \nc\rmL{\mathrm{L}} \nc\rml{\mathrm{l}}
\nc\bfm{{\bf{m}}}\nc\bfM{{\boldsymbol M}}\nc\cM{{\cal M}} \nc\fM[1]{M\br*{#1}} \nc\fm[1]{m\br*{#1}} \nc\rmM{\mathrm{M}} \nc\rmm{\mathrm{m}}
\nc\bfn{{\bf{n}}}\nc\bfN{{\boldsymbol N}}\nc\cN{{\cal N}} \nc\fN[1]{N\br*{#1}} \nc\fn[1]{n\br*{#1}} \nc\rmN{\mathrm{N}} \nc\rmn{\mathrm{n}}
\nc\bfo{{\bf{o}}}\nc\bfO{{\boldsymbol O}}\nc\cO{{\cal O}} \nc\fO[1]{O\br*{#1}} \nc\fo[1]{o\br*{#1}} \nc\rmO{\mathrm{O}} \nc\rmo{\mathrm{o}}
\nc\bfp{{\bf{p}}}\nc\bfP{{\boldsymbol P}}\nc\cP{{\cal P}} \nc\fP[1]{P\br*{#1}} \nc\fp[1]{p\br*{#1}} \nc\rmP{\mathrm{P}} \nc\rmp{\mathrm{p}}
\nc\bfq{{\bf{q}}}\nc\bfQ{{\boldsymbol Q}}\nc\cQ{{\cal Q}} \nc\fQ[1]{Q\br*{#1}} \nc\fq[1]{q\br*{#1}} \nc\rmQ{\mathrm{Q}} \nc\rmq{\mathrm{q}}
\nc\bfr{{\bf{r}}}\nc\bfR{{\boldsymbol R}}\nc\cR{{\cal R}} \nc\fR[1]{R\br*{#1}} \nc\fr[1]{r\br*{#1}} \nc\rmR{\mathrm{R}} \nc\rmr{\mathrm{r}}
\nc\bfs{{\bf{s}}}\nc\bfS{{\boldsymbol S}}\nc\cS{{\cal S}} \nc\fS[1]{S\br*{#1}} \nc\fs[1]{s\br*{#1}} \nc\rmS{\mathrm{S}} \nc\rms{\mathrm{s}}
\nc\bft{{\bf{t}}}\nc\bfT{{\boldsymbol T}}\nc\cT{{\cal T}} \nc\fT[1]{T\br*{#1}} \nc\ft[1]{t\br*{#1}} \nc\rmT{\mathrm{T}} \nc\rmt{\mathrm{t}}
\nc\bfu{{\bf{u}}}\nc\bfU{{\boldsymbol U}}\nc\cU{{\cal U}} \nc\fU[1]{U\br*{#1}} \nc\fu[1]{u\br*{#1}} \nc\rmU{\mathrm{U}} \nc\rmu{\mathrm{u}}
\nc\bfv{{\bf{v}}}\nc\bfV{{\boldsymbol V}}\nc\cV{{\cal V}} \nc\fV[1]{V\br*{#1}} \nc\fv[1]{v\br*{#1}} \nc\rmV{\mathrm{V}} \nc\rmv{\mathrm{v}}
\nc\bfw{{\bf{w}}}\nc\bfW{{\boldsymbol W}}\nc\cW{{\cal W}} \nc\fW[1]{W\br*{#1}} \nc\fw[1]{w\br*{#1}} \nc\rmW{\mathrm{W}} \nc\rmw{\mathrm{w}}
\nc\bfx{{\bf{x}}}\nc\bfX{{\boldsymbol X}}\nc\cX{{\cal X}} \nc\fX[1]{X\br*{#1}} \nc\fx[1]{x\br*{#1}} \nc\rmX{\mathrm{X}} \nc\rmx{\mathrm{x}}
\nc\bfy{{\bf{y}}}\nc\bfY{{\boldsymbol Y}}\nc\cY{{\cal Y}} \nc\fY[1]{Y\br*{#1}} \nc\fy[1]{y\br*{#1}} \nc\rmY{\mathrm{Y}} \nc\rmy{\mathrm{y}}
\nc\bfz{{\bf{z}}}\nc\bfZ{{\boldsymbol Z}}\nc\cZ{{\cal Z}} \nc\fZ[1]{Z\br*{#1}} \nc\fz[1]{z\br*{#1}} \nc\rmZ{\mathrm{Z}} \nc\rmz{\mathrm{z}}

\DeclareMathOperator{\Supp}{supp}

\nc\defeq{\coloneqq}

\newcommand{\supp}[1]{\Supp\br*{#1}}

\DeclarePairedDelimiterX\Set[1]\{\}{#1}

\newcommand\R{{\mathbb R}}

\newtheorem{thm}{Theorem}

\newtheorem{lemma}{Lemma}
\newtheorem*{lemma*}{Lemma}

\newtheorem{corollary*}[theorem]{Corollary}

\newtheorem*{remark*}{Remark}

\title{Support Recovery of Sparse Signals from a Mixture of Linear Measurements}

\author{~~~Venkata~Gandikota\footnote{V. Gandikota is with the Electrical Engineering \& Computer Science Department at the Syracuse University, NY13210, USA (email: \texttt{gandikota.venkata@gmail.com}).}~~~Arya~Mazumdar\footnote{A.Mazumdar is with the Halıcıoğlu Data Science Institute at University of California, San Diego USA (email: \texttt{arya@ucsd.edu}).} ~~~Soumyabrata~Pal\footnote{S. Pal is with the Computer Science Department at the University of Massachusetts Amherst, Amherst, MA 01003, USA (email: \texttt{spal@cs.umass.edu}).}}

\begin{document}

\maketitle
\begin{abstract}
Recovery of support of a sparse vector from simple measurements is a widely-studied problem, considered under the frameworks of compressed sensing, 1-bit compressed sensing, and more general single index models. We consider generalizations of this problem: mixtures of linear regressions, and mixtures of linear classifiers, where the goal is to recover supports of multiple sparse vectors using only a small number of possibly noisy linear, and 1-bit measurements respectively. The key challenge is that the measurements from different vectors are randomly mixed. Both of these problems have also received attention recently. In mixtures of linear classifiers, the observations correspond to the side of queried hyperplane a random unknown vector lies in, whereas in mixtures of linear regressions we observe the projection of a random unknown vector on the queried hyperplane. The primary step in recovering the  unknown vectors from the mixture is to first identify the support of all the individual component vectors. In this work we study the number of measurements sufficient for recovering the supports of all the component vectors in a mixture in both these models. We provide algorithms that use a number of measurements polynomial in $k, \log n$ and 
%only $\s{poly}(k, \log n)$  and 
quasi-polynomial in $\ell$, to recover the support of all the $\ell$ unknown vectors in the mixture with high probability when each individual component is a $k$-sparse $n$-dimensional vector.
\end{abstract}

\section{Introduction}

In  the support recovery problem, widely studied in the literature of compressed sensing \cite{blumensath2009iterative,aeron2010information,reeves2009note}, the objective is to recover the support (positions of nonzero coordinates) of a sparse vector from minimal number of (noisy) {\em linear} measurements. The support recovery problem is also extensively studied under the 1-bit compressed sensing model where measurements are further quantized and just the signs of the linear measurements are provided~\cite{gopi2013one,jacques2013robust,acharya2017improved}.

In a recent line of work that started with \cite{yin2018learning}, a generalization of the sparse recovery problem is considered \cite{kris2019sampling,mazumdar2020recovery,gandikota2020recovery,chen2020learning}, where instead of one sparse vector, multiple unknown sparse vectors are to be recovered. However any attempt to obtain a measurement (linear or 1-bit) from the vectors results in a mixture model, where a vector from the unknown set is picked uniformly to generate the response. Due to the asynchronicity of the measurements, this set of problems pose fundamentally different challenges than recovery of a single sparse vector.

This line of work also connects the mixture of simple learning models that have been studied extensively in  the past few decades,  with mixtures of linear regression model being more widely studied~\cite{de1989mixtures,chen2014convex, huang2013nonparametric, khalili2007variable, shen2019iterative, song2014robust,wang2019convergence, yi2016solving, zhu2004hypothesis} than mixture of linear classifiers \cite{sun2014learning,sedghi2016provable}. Such mixture models, that assume the training data to come from multiple models, are good approximators of a function \cite{bishop1998latent,jordan1994hierarchical} and have numerous applications in modeling heterogeneous settings such as   machine translation~\cite{liang2006end}, behavioral health~\cite{deb2000estimates}, medicine~\cite{blackwell2006applying}, object recognition~\cite{quattoni2005conditional} etc.

%Two of the most canonical problems in Machine Learning is to train a linear classifier on data with binary output and to train a linear regressor on data with numeric output.
The mixture of sparse recovery models of \cite{yin2018learning} and followup works can be 
framed as a Mixture of Linear Classifiers (MLC) or a Mixture of Linear Regressions (MLR) problems. The statistical model in MLC is the following: there exists $\ell$ unknown hyperplanes with normal vectors $\f{v}^1,\f{v}^2,\dots,\f{v}^{\ell}$ and for a particular feature vector, the label (response) is generated stochastically by selecting one of the unknown hyperplanes at random and then returning the side of the chosen hyperplane on which the feature vector lies. In MLR, the statistical model again assumes the presence of $\ell$ unknown hyperplanes with normals $ \f{v}^1,\f{v}^2,\dots,\f{v}^{\ell}$ and for a particular feature vector, the response is stochastically generated by selecting one of the unknown hyperplanes at random and then returning the projection of the feature vector to the chosen hyperplane. In order to make these models more general, we can assume that the responses are corrupted by noise. The overarching goal  for both the MLR and MLC is to learn the $\ell$ unknown hyperplanes as accurately as possible, using the least number of noisy responses. Sparsity, incorporated into the MLC and MLR problems, is also a common assumption that represents redundant features and lower dimensionality of the models~\cite{zhang2014lower,reeves2019all,chan2007direct}.  

%Furthermore, in recent times, there has been a deluge of data because of which Machine Learning algorithms have to cope with high dimensional datasets with many irrelevant features. In such cases, an important pre-processing step is feature selection or dimensionality reduction in which the objective is to first select a small number of important features and subsequently, train ML algorithms based on those features. We model this fact statistically in the MLR/MLC setting by assuming that the unknown vectors $\f{v}^1,\f{v}^2,\dots,\f{v}^{\ell}$ are $k$-sparse i.e. they have at most $k$ non-zero values and therefore at most $k$ features are relevant for each hyperplane. Thus feature selection in this model reduces to recovering the support of the unknown vectors corresponding to each hyperplane using minimum number of noisy responses. 

In this work, we tackle the problem of  \textit{support recovery} of sparse vectors for both MLR and MLC model in the active query based setting of~\cite{yin2018learning,kris2019sampling,mazumdar2020recovery,gandikota2020recovery}. %In this setting, we can design a feature; query the designed feature to the oracle and obtain the corresponding response according to the statistical model. To restate our objective, 
Our goal is to recover the support of all the unknown sparse vectors (hyperplane normals) with minimum number of measurements.

\subsection{Formal Problem Statement and Relevant Works}
In both the problems below, let $\ca{V}$ be a set of $\ell$ unknown vectors $\f{v}^1,\f{v}^2, \dots,\f{v}^{\ell} \in \bb{R}^n$ such that $\left|\left|\f{v}^i\right|\right|_0 \le k$ for all $i \in [\ell]\equiv \{1,2, \dots, \ell\}$. 
\paragraph{Mixtures of Sparse Linear Classifiers (MLC).}

%In this setting, we have an 
%\subsubsection{Query Model}
Let $\s{sign}:\bb{R} \rightarrow \{-1,+1\}$ be the sign function %$\s{sign}(x)=\mathds{1}[x \ge 0]-\mathds{1}[x < 0]$ 
that takes a real number and returns its sign. 
We consider MLC label queries $\ca{O}:\bb{R}^n \rightarrow \{-1,+1\}$ that takes as input a query vector $\f{x} \in \bb{R}^n$ and returns %as output the quantity 
\begin{align*}
\s{sign}(\langle \f{x},\f{v} \rangle) \cdot (1-2Z)
\end{align*}
where $\f{v}$ is sampled uniformly at random from $\ca{V}$ and $Z \sim \s{Ber}(\eta)$, the noise, is a Bernoulli random variable that is $1$ with probability $\eta$ and $0$ with probability $1-\eta$. 
%For any vector $\f{v} \in \bb{R}^n$, let us define the support of $\f{v}$ to be a binary vector $\s{supp}(\f{v})\in \{0,1\}^n$ such that $\s{supp}(\f{v})_i = 1$ if $\f{v}_i \neq 0$ and $0$ otherwise. 
In this problem, our objective is to recover the support of all the unknown vectors in $\ca{V}$ using minimum number of label queries. 
%{\color{red}This statement was commented out before. Should it be?}

The only relevant work in this setting is \cite{gandikota2020recovery} which provided results for both support recovery and approximate recovery of the unknown vectors. However, the results of \cite{gandikota2020recovery} are valid only  under the restrictive assumption that the support of any unknown vector is not contained within the union of the supports of the other unknown vectors. 

In this work, we generalize the techniques of \cite{gandikota2020recovery} for support recovery of the unknown vectors and get rid of the restrictive assumption. We further improve the generalized result in a wide regime by demonstrating a new low-rank tensor decomposition based algorithm for support recovery.

\paragraph{Mixtures of Sparse Linear Regressions (MLR).}
% We assume that the magnitude of any non-zero entry of any unknown vector in $\ca{V}$ is at least $\delta$ i.e. 
%\begin{align*}
%    \min_{i \in [\ell]}\min_{j \in [n]: \f{v}^i_j \neq 0} |\f{v}^i_j| \ge \delta. 
%\end{align*}
In this setting, we have an MLR  label map $\ca{O}: \bb{R}^n \rightarrow \bb{R}$ that takes as input a query  $\f{x} \in \bb{R}^n$ and returns as output the quantity 
$$
\langle \f{x},\f{v} \rangle +Z
$$
where $\f{v}$ is sampled uniformly at random from $\ca{V}$ and $Z \sim \ca{N}(0,\sigma^2)$ is a zero-mean Gaussian random variable with variance $\sigma^2$. For our MLR   results to hold we further assume that the minimum magnitude of any non-zero entry of any unknown vector in $\ca{V}$ is known to be at least $\delta$, i.e.,  $ \min_{i \in [\ell]}\min_{j \in [n]: \f{v}^i_j \neq 0} |\f{v}^i_j|\ge \delta$.

Note that, because of the additive noise, a result for MLC setting cannot be transformed into a result in MLR setting (i.e., MLC response is not simple quantization of MLR).   % and not multiplicative.

It is possible to increase the $\ell_2$ norm of the queries arbitrarily so that the noise becomes inconsequential. To avoid this we use the following definition of signal to noise ratio. Suppose the algorithm designs the $i^{th}$ query vector by first choosing a distribution $Q^i$ and subsequently sampling a query vector $\mathbf{x}^i \sim Q^i$. The signal to noise ratio is defined as follows:
\begin{align}\label{eq:snr}
\mathsf{SNR}= \max_{i} \min_{j\in[\ell]} \frac{\bb{E}_{\mathbf{x}^i \sim Q^i} |\langle \mathbf{x}^i,\mathbf{\f{v}}^{\ell} \rangle|^{2}}{\bb{E} Z^2} \ .
\end{align}
Our objective in this setting is to recover the support of all unknown vectors $\f{v}^1,\f{v}^2,\dots,\f{v}^{\ell} \in \mathbb{R}^n$ while minimizing the number of queries for a fixed $\mathsf{SNR}$.

%%%%%%%%%%%%%%%%%%%%%%%%%%%%%%%%%%%%%%%%%%%%%%%%

%\subsection{Most Relevant Work and Main Contributions}

%\paragraph{Mixture of Linear Regressions:} 
The most relevant works in this setting would be \cite{yin2018learning}, \cite{kris2019sampling} and \cite{mazumdar2020recovery}, all of which were concerned with approximately recovering the $k$-sparse unknown vectors $\f{v}^1,\f{v}^2,\dots,\f{v}^{\ell}$ i.e. computing estimates $\hat{\f{v}}^1,\hat{\f{v}}^2,\dots,\hat{\f{v}}^{\ell}$ such that for some precision parameter $\gamma>0$,
\begin{align*}
    \|\f{v}^i-\hat{\f{v}}^{\sigma(i)}\|_2 \le O(\gamma) \quad \text{for all } i\in [\ell]
\end{align*}
for some permutation $\sigma:[\ell]\rightarrow [\ell]$.
While approximate recovery of vectors can  also be translated into support recovery, the results of  \cite{yin2018learning} and \cite{kris2019sampling} are valid only under the restrictive  assumption that the sparse vectors all belong to some scaled integer lattice. The result of \cite{mazumdar2020recovery} does not have any restriction, but it holds only when $\ell=2.$ However, note that in this special case i.e. when $\ell=2$, \cite{mazumdar2020recovery} provides a query complexity guarantee that is linear in the sparsity $k$. On the other hand, our query complexity guarantees (see Section \ref{sec:results}) have a polynomial dependence on $k$ (with a larger degree) implying that in the regime when $\ell=2$ and $k$ is large, the guarantees of \cite{mazumdar2020recovery} are better.

%Learning the unknown vectors up to a precision of $\gamma$ is evidently a harder objective than just learning the support if we assume that the non-zero entries of the unknown vectors are large enough. However, \cite{yin2018learning} provide results under some particularly severe restrictive conditions on the unknown vectors:
%\begin{itemize}
%\item Any pair of unknown vectors must have distinct values in any index lying in the intersection of their supports.
%\item for some $\epsilon >0$ , the entries of all the unknown vectors must be some integer multiple of $\epsilon$.
%\end{itemize}
%Although \cite{kris2019sampling} got rid of the first assumption, they fail to get rid of the second assumption. This is unfavorable as their sample complexity guarantees were exponential in $1/\epsilon$.  Furthermore, the sample complexity guarantees of \cite{yin2018learning,kris2019sampling} were exponential in the variance of the noise and the number of unknown vectors $\ell$. \cite{mazumdar2020recovery} got rid of all assumptions and had a polynomial dependence on the noise variance but their results were restricted to the special case of two unknown vectors ($\ell=2$).

Here we provide results for support recovery of any number of unknown vectors that do not have any of the aforementioned restrictions and also have a polynomial dependence on the noise variance, sparsity and a near polynomial dependence on the number of unknown vectors.

%\paragraph{Mixture of Linear Classifiers:}  

\subsection{Other Related Work}

Learning the unknown vectors in the MLR setting is a generalization of the  compressed sensing problem~\cite{candes2006robust,donoho2006compressed} where the objective is to learn a single unknown $k$-sparse vector ($\ell=1$) with minimum number of noisy linear measurements. % We refer to an excellent survey by \cite{boche2015survey} for compressed sensing results (in particular the results of \cite{candes2008restricted} and \cite{baraniuk2008simple} are useful). 
  Support recovery is a well-studied area within  this literature~\cite{blumensath2009iterative,aeron2010information,reeves2009note}.
Similarly, learning the unknown vectors in the MLC setting is a generalization of the 1-bit compressed sensing problem  where the objective is to learn a single unknown $k$-sparse vector ($\ell=1$) with minimum number of linear measurements quantized to only $1$-bit. Support recovery of the sparse vector from 1-bit measurements has also been widely studied \cite{acharya2017improved, gopi2013one,flodin2019superset,jacques2013robust}. %A large part of the literature focuses on the non-adaptive setting for the problem \cite{acharya2017improved, ai2014one, flodin2019superset, gopi2013one, JLBB13, plan2013one}. 

The major building block of one of our two algorithms is low-rank tensor decomposition also known as Canonical Polyadic (CP) decomposition. Tensor decomposition has been widely used in parameter estimation in mixture models and latent variable models. We refer the reader to \cite{rabanser2017introduction} and the references therein for a detailed survey. Our other algorithm makes use of combinatorial structures such as a general class of Union Free Families (UFF), see, \cite{stinson2004generalized}, to recover the support. UFFs have been previously used in \cite{acharya2017improved} and \cite{gandikota2020recovery} for support recovery in  linear classifiers.

\noindent{\bf Organization.}
The rest of the paper is organized as follows. 
In Section \ref{sec:prelim}, we gave the necessary backgrounds, and described our techniques and main results, namely, Theorems \ref{lem:t-iden-supp-rec}, \ref{lem:s-lin-indep-supp-rec}, \ref{lem:s-indep-supp-rec}, and Corollary \ref{coro:imp}.
In Section \ref{sec:detailed}, we provided the detailed proofs of Theorem~\ref{lem:t-iden-supp-rec} (Section \ref{subsec:p_identifiable}), Theorem~\ref{lem:s-lin-indep-supp-rec} (Section \ref{subsec:flip-independent}) and Theorem~\ref{lem:s-indep-supp-rec} (Section \ref{subsec:r-kruskal}) while deferring the proof of Theorem \ref{lem:suff-t} to Section \ref{app:prooflog}. In Sections \ref{sec:compute-SuCa} and \ref{sec:Sunion}, we give the details of a Lemma that is an integral component of the proofs of our main Theorems. We delegate the discussion on Jennrich's algorithm to the Appendix.  %Finally, due to space limitations we have deferred the Conclusion to Appendix \ref{app:conc}.

\section{Our Techniques and Results}
\label{sec:prelim}
\subsection{Preliminaries}\label{subsec:tensor}

\paragraph{Notations:} 
Let $\s{round}:\bb{R}\rightarrow\bb{Z}$ denote a function that returns the closest integer to a given real input.
Let $\f{1}_n$ denote a length $n$ vector of all $1$'s. We will write $[n]$ to denote the set $\{1, \ldots, n\}$ and let $\ca{P}([n])$ be the power set of $[n]$. For a vector ${\bf v} \in \R^n$, let $\f{v}_i$ denote its $i$-th coordinate for any $i\in [n]$. We will use $\supp{\bf v} \subseteq [n]$ to denote the support of the vector $\f{v}$, i.e, the set of indices with non-zero entries in $\f{v}$. We will abuse notations a bit, and also sometimes use $\supp{\f{v}}$ to denote the binary indicator vector of length $n$ that takes $1$ at index $i$ if and only if $\f{v}_i \neq 0$. For a vector $\f{v}\in \bb{R}^n$ and subset $S \subseteq [n]$ of indices, let $\f{v}|_S\in  \R^{|S|}$ denote the vector $\f{v}$ restricted to the indices in $S$. Finally, let $f:\ca{P}([n])\times \{0,1\}^n \rightarrow \{0,1\}^n$ be a function that takes a binary vector $\f{v}\in \{0,1\}^n$ and a subset $\ca{S}\subseteq [n]$ as input and returns another binary vector $\f{v}'$ such that the indices of $\f{v}$ corresponding to the the set $\ca{S}$ are flipped i.e. $\f{v}'_i=\f{v}_i \oplus 1$ if $i \in \ca{S}$ and $\f{v}'_i=\f{v}_i$ otherwise.
\paragraph{Tensor Decomposition:} 
Consider a tensor $\ca{A}$ of order $w \in \bb{N}, w>2$ on $\bb{R}^n$ which is denoted by $\ca{A} \in \bb{R}^n \otimes \bb{R}^n \otimes \dots \otimes \bb{R}^n \; (w \; \text{times})$.
Let us denote by $\ca{A}_{i_1,i_2,\dots,i_{w}}$ where $i_1,i_2,\dots,i_{w} \in [n],$  the element in $\ca{A}$ whose location along the $j^{\s{th}}$ dimension is $i_j$ i.e. there are $i_j-1$ elements along the $j^{\s{th}}$ dimension before $\ca{A}_{i_1,i_2,...,i_w}$ . Notice that this indexing protocol uniquely determines the element within the tensor.  For a  detailed review of tensors, we defer the reader to \cite{kolda2009tensor}.
In this work, we are interested in low rank decomposition of tensors. 
A tensor $\ca{A}$ can be described as a rank-\texttt{1} symmetric tensor if it can be expressed as
%\footnote{In this work we only consider the special case of rank-1 tensors where every component is identical. In general, rank-1 tensors can be described as $\ca{A} = \f{z}^1\otimes \f{z}^2 \otimes \dots \otimes \f{z}^w$ for $\f z^1, \ldots \f z^w \in \bb{R}^n$.} 
\begin{align*}
    \ca{A} =\underbrace{\f{z}\otimes \f{z} \otimes \dots \otimes \f{z}}_{w \text{\;\; times}}
\end{align*}
for some $\f{z}\in \bb{R}^n$ i.e. $\ca{A}_{i_1,i_2,\dots,i_{w}} = \prod_{j=1}^{w}\f{z}_{i_j}$.  % i.e. 
%\begin{align*}
%    \ca{A}=\sum_{r=1}^{R}\underbrace{\f{z}^r\otimes \f{z}^r \otimes \dots \otimes \f{z}^r}_{w \text{\;\; times}}.
%\end{align*}
A tensor $\ca{A}$ that can be expressed
as a sum of $R$ rank-\texttt{1} symmetric tensors is defined as a rank $R$ symmetric tensor. For such a rank $R$ tensor $\ca{A}$ provided as input, we are concerned with the problem of unique decomposition of $\ca{A}$ into a sum of $R$ rank-\texttt{1} symmetric tensors; such a decomposition is also known as a Canonical Polyadic (CP) decomposition. Below, we show a result due to \cite{sidiropoulos2000uniqueness} describing the sufficient conditions (Kruskal's result) for the unique CP decomposition of a rank $R$ tensor $\ca{A}$:
\begin{lemma}[Unique CP decomposition \cite{sidiropoulos2000uniqueness}]\label{lem:unique_cp}
Suppose $\ca{A}$ is the sum of $R$ rank-one tensors i.e. 
%\begin{align*}
\[
    \ca{A}=\sum_{r=1}^{R}\underbrace{\f{z}^r\otimes \f{z}^r \otimes \dots \otimes \f{z}^r}_{w \text{\;\; times}}.
\]
%\end{align*}
and further, the Kruskal Rank of the $n \times R$ matrix whose columns are formed by  $\f{z}^1,\f{z}^2,\dots,\f{z}^R$ is $J$. Then, if 
$
    wJ \ge 2R+(w-1),
$
then the CP decomposition is unique and we can recover the vectors $\f{z}^1,\f{z}^2,\dots,\f{z}^R$ up to permutations.
\end{lemma}

There exist many different techniques for CP decomposition of a tensor but the most well-studied ones are Jennrich's Algorithm (see Section 3.3, \cite{moitra2014algorithmic}) and the Alternating Least Squares (ALS) algorithm \cite{kolda2009tensor}. Among these, Jennrich's algorithm (see Section \ref{sec:jenn} for more details) is efficient and recovers the latent rank-1 tensors uniquely but it works only for tensors of order $3$ when the underlying vectors $\f{z}^1,\f{z}^2,\dots,\f{z}^R$ are linearly independent (See Theorem 3.3.2, \cite{moitra2014algorithmic}); this is a stronger condition than what we obtain from Lemma \ref{lem:unique_cp} for $w=3$.
 On the other hand, the ALS algorithm is an iterative algorithm which is easy to implement for tensors of any order but unfortunately, it takes many iterations to converge and furthermore, it is not guaranteed to converge to the correct solution. Jennrich's algorithm also has the additional advantage that it will throw an error if its sufficient condition for unique CP decomposition %($\f{z}^1,\f{z}^2,\dots,\f{z}^R$ being linearly independent)  for $w=3$ 
 is not satisfied. This property will turn out to be useful for the problem that we study in this work. Finally, notice that if $\ca{A}$ is the weighted sum of $R$ rank-1 tensors i.e.,
%\begin{align*}
 \[   \ca{A}=\sum_{r=1}^{R}\lambda_r\underbrace{\f{z}^r\otimes \f{z}^r \otimes \dots \otimes \f{z}^r}_{w \text{\;\; times}}.
 \]
%\end{align*}
then we can rewrite $\ca{A}=\sum_{r=1}^{R}\f{y}^r\otimes \f{y}^r \otimes \dots \otimes \f{y}^r $ where $\f{y}^r = \lambda_r^{1/w}\f{z}^r$. If $\{\f{y}^r\}_{r=1}^{R}$ satisfies the conditions of Lemma \ref{lem:unique_cp} and if it is known that $\{\f{z}^r\}_{r=1}^{R}$ are binary vectors, then we can still recover $\f{z}^r$ by first computing $\f{y}^r$ and then taking its support for all $r \in [R]$. Subsequently, notice that we can also recover $\{\lambda_r\}_{r=1}^{R}$. As we discussed, for tensors of order $w>3$, there is no known efficient algorithm that can  recover the correct solution even if its existence and uniqueness is known. Due to this limitation, it was necessary in prior works using low rank decomposition of tensors that the unknown parameter vectors are linearly independent \cite{chaganty2013spectral,anandkumar2014tensor} since tensors of order $>3$ could not be used. However, if it is known apriori that the vectors $\{\f{z}^r\}_{r=1}^{R}$ are binary and the coefficients $\{\lambda_r\}_{r=1}^{R}$ are positive integers bounded from above by some $C>0$, then we can exhaustively search over all possibilities ($O(C2^n)$ of them) to find the unique decomposition even in higher order tensors. The set of possible solutions can be reduced significantly if the unknown vectors are known to be sparse as is true in our setting.

\paragraph{Family of sets:}
We now review literature on some important families of sets called \emph{union free families} \cite{erdos1985families} and \emph{cover free families} \cite{kautz1964nonrandom} that found applications in cryptography, group testing and 1-bit compressed sensing. These special families of sets are used crucially in this work. 

\begin{defn}
[Robust Union Free Family $(d,t,\alpha)$- $\s{RUFF}$~\cite{acharya2017improved}]
Let $d, t$ be integers and $0 \le \alpha \le 1$. A family of sets $\ca{F}=\{\ca{H}_1,\ca{H}_2,\dots,\ca{H}_n\}$ with each $\ca{H}_i \subseteq [m]$ and $|\ca{H}|=d$ is a $(d,t,\alpha)$-$\s{RUFF}$ if for any set of $t$ indices $T \subset [n], |T| = t$, and any index $j \notin T$, 
$
\left| \ca{H}_{j} \setminus \left( \bigcup_{i \in T} \ca{H}_{i} \right) \right| > (1-\alpha) d.
$
\end{defn}
We refer to $n$ as the size of the family of sets, and $m$ to be the alphabet over which the sets are defined. $\s{RUFF}$s were studied earlier in the context of support recovery of 1bCS \cite{acharya2017improved}, and a simple randomized construction of $(d,t,\alpha)$-$\s{RUFF}$ with $m = O(t^2 \log n)$ was proposed by De Wolf \cite{de2012efficient}.

\begin{lemma}{\cite{acharya2017improved,de2012efficient}}
{\label{lem:ruffexist}}
Given $n, t$ and $\alpha > 0$, there exists an $(d,t,\alpha)$-$\s{RUFF}$, $\ca{F}$ with $m=O\big((t^2\log n)/\alpha^2)$ and $d=O((t \log n)/\alpha)$.
\end{lemma}

$\s{RUFF}$ is a generalization of the family of sets known as the Union Free Familes ($\s{UFF}$) - which are essentially $(d,t,1)$-$\s{RUFF}$. We require yet another generalization of $\s{UFF}$ known as Cover Free Families ($\s{CFF}$) that are also sometimes referred to as  superimposed codes \cite{d2014bounds}. 

\begin{defn}[Cover Free Family $(r, t)$-$\s{CFF}$]
A family of sets $\ca{F}=\{\ca{H}_1,\ca{H}_2,\dots,\ca{H}_n\}$ where each $\ca{H}_i \subseteq [m]$ is an $(r, t)$-$\s{CFF}$ if for any pair of disjoint sets of indices $T_1, T_2 \subset [n]$ such that $|T_1| = r, |T_2| = t, T_1 \cap T_2 = \emptyset$, 
%distinct $j_0,j_1,\dots,j_{k+1} \in [n]$, 
$\left| \bigcap_{i \in T_1} \ca{H}_{i}  \setminus \bigcup_{i  \in T_2} \ca{H}_{i} \right| > 0.
$ 
\end{defn}

Several constructions and bounds on existence of $\s{CFF}$s are known in literature. We state the following lemma regarding the existence of $\s{CFF}$ which can be found in \cite{ruszinko1994upper, furedi1996onr}. We also include a proof in the supplementary material  for the sake of completeness. 
\begin{lemma}\label{lem:cffexist}
For any given integers $r, t$, there exists an $(r,t)$-$\s{CFF}$, $\ca{F}$ of size $n$ with $m=O(t^{r+1}\log n)$.
\end{lemma}

\subsection{Our Techniques}
%For a set $\mathcal{V}$ of $\ell$ unknown vectors, let ${\bf A} \in \{0,1\}^{n \times \ell}$ denote their support matrix where each column vector $\f{A}_i \in \{0,1\}^n$ represents the support of the unknown vector $\f{v}^i$.

Recall that the set of unknown vectors is denoted by $\ca{V}\equiv \{\f{v}^1,\f{v}^2,\dots,\f{v}^{\ell}\}$.
Let ${\bf A} \in \{0,1\}^{n \times \ell}$ denote the support matrix corresponding to $\ca{V}$ where each column vector $\f{A}_i \in \{0,1\}^n$ represents the support of the $i^{\s{th}}$ unknown vector $\f{v}^i$.
For any ordered tuple $C \subset [n]$ of indices, and any binary string ${\bf a} \in \{0,1\}^{|C|}$, define $\Su(C, {\bf a})$ to be the set of unknown vectors whose corresponding supports have the substring ${\bf a}$ at positions indexed by $C$,
i.e., 
\[
\Su(C, \f{a}) := \{ \f{v}^i \in \mathcal{V} \mid \supp{\f{v}^i}|_C = \bf{a} \}.
\]

In order to describe our techniques and our results, we need to introduce three different properties of matrices and extend them to a set of vectors by using their corresponding support matrix. The proofs of our main results  follow from the guarantees of algorithms (Algorithm~\ref{algo:t-iden-supp-rec},  Algorithm~\ref{algo:s-lin-indep-supp-rec} and Algorithm~\ref{algo:s-indep-supp-rec}) each of which leverage the aforementioned key properties of the unknown support matrix $\f{A}$. While explaining the intuition behind the introduced matrix properties, we will assume that all the unknown vectors in $\ca{V}$ have distinct supports for simplicity. However, this assumption is not necessary and the guarantees of all our algorithms hold even when the supports are not distinct albeit with slightly more involved arguments (see Section \ref{sec:detailed}).

\begin{defn}[$p$-identifiable]\label{def:iden}
The $i^{\s{th}}$ column $\f{A}_i$ of  a binary matrix $\f{A} \in \{0,1\}^{n \times \ell}$ with all distinct columns is called $p$-identifiable if % for some column $\f{A}_i$,
there exists a %unique tuple $(C, \bf a) \in [n]^t \times \{0,1\}^t $, where 
  set $S \subset [n]$ of at most $p$-indices and a binary string $\f{a} \in \{0,1\}^p$ such that $\f{A}_i|_S = \f a$, and $\f{A}_j|_S \neq \f a$ for all $j \neq i$.
  
  A binary matrix $\f{A} \in \{0,1\}^{n \times \ell}$ with all distinct columns is called  $p$-identifiable if there exists a permutation
$\sigma:[\ell]\rightarrow[\ell]$
such that for all $i\in [\ell]$, the sub-matrix $\f{A}^i$ formed by deleting the columns indexed by the set $\{\sigma(1),\sigma(2),\dots,\sigma(i-1)\}$ has at least one $p$-identifiable column.

Let $\mathcal{V}$ be set of $\ell$ unknown vectors in $\bb{R}^n$, and $\f{A} \in \{0,1\}^{n \times \ell}$ be its support matrix. Let $\f{B}$ be the matrix obtained by deleting duplicate columns of $\f{A}$. The set $\mathcal{V}$ is called $p$-identifiable if  $\f{B}$ is $p$-identifiable. 
\end{defn}

\paragraph{Support matrix $\f{A}$ is $p$-identifiable:}

Algorithm~\ref{algo:t-iden-supp-rec} uses the property that the support matrix $\f{A}$ is $p$-identifiable for some known $p\le \log \ell$ (See Theorem \ref{lem:suff-t}) to recover the supports of all the unknown vectors. First, we briefly describe the support recovery algorithm of \cite{gandikota2020recovery} where the authors crucially use the \emph{separability} of supports of the unknown vectors to recover them. Their algorithm assumes that the support of each unknown vector contains a unique identifying index, i.e., 
%We first briefly describe the result of \cite{gandikota2020recovery} and the observations that will allow us to extend the results without the separability assumption.  
%Let $\Su(i, 1)$ denote the set of unknown vectors that are supported on the $i$-th coordinate. 
%the separability assumption ensures that 
for each unknown vector $\bf v \in \ca{V}$, there exists a unique index $i \in [n]$ such that $\Su((i), 1) = \{\bf v\}$, and hence $|\Su((i), 1)|=1$. Observe that if $|\Su((i), 1)|=1$, and $|\Su((i,j), (1,1))|=1$ for some $i \neq j$, then it follows that both the indices $i, j$ belong to the support of the same unknown vector. Therefore \cite{gandikota2020recovery} are able to recover the supports of all the unknown vectors by computing $|\Su((i), 1)|$ and $|\Su((i,j), (1,1))|$ for all $i, j \in [n]$.  The crux of their algorithm lies in computing all the $n$ values of $|\Su((i), 1)|$, and $O(n^2)$ values of  $|\Su((i,j), (1,1))|$ using just $\s{poly}(\ell, k)$ queries. 
%For the indices in set $T := \{ i \in [n] |~|\Su(i, 1)|=1\}$
%to denote all the unique coordinates.  For all $i \in T$, 
%and all $j \in [n]$, \cite{gandikota2020recovery} compute $|\Su((i,j), (1,1))|$, where $\Su((i,j), (1,1))$ denotes the set of all unknown vectors that are supported at both indices $i,j$. Observe that for any $i \in T$, the cardinally of the set $|S((i,j), (1,1))| \in \{0,1\}$, and using this information, \cite{gandikota2020recovery} can recover the support of the unknown vector $\bf v$ such that $S(i, 1) = \{\bf v\}$. 
We can generalize the support recovery technique of \cite{gandikota2020recovery} by observing that if $\f{A}$ is $p$-identifiable, then there exists at least one unknown vector $\f{v}\in \ca{V}$ that has a unique sub-string of length at most $p$. 
Hence, there exists a unique set $C\subseteq [n]$ and string $\f{a}\in \{0,1\}^{|C|}$ satisfying $|C|\le p$ such that $\Su(C \cup \{j\}, (\f{a}, 1)) = \{\f{v}\}$. By a similar observation as before, if  $|\Su(C \cup \{j\}, (\f{a}, 1))| = 1$ for some $j \in [n] \setminus C$, we can certify that $j \in \supp{\f v}$ and if $|\Su(C \cup \{j\}, (\f{a}, 1))| = 0$, then $j \not \in \supp{\f v}$. Hence we can reconstruct the support of $\f{v}$ and subsequently, we can update  $|\s{occ}(C,\f{a})|\leftarrow |\s{occ}(C,\f{a})|-\f{1}[\supp{\f  v}_{\mid \ca{C}} = \f{a}]$ for all sets $\ca{C}$ satisfying $|\ca{C}|\le p$ and all $\f{a}\in \{0,1\}^{|\ca{C}|}$. Note that the updated $\s{occ}$ values correspond to the support matrix $\f{A}$ excluding the column corresponding to the support of $\f{v}$. From the definition of $p-$identifiable, we can recursively apply the same arguments as above and recover the support vectors one by one. 
The main technical challenge then lies in computing all the $O(2^p n^p)$ values $|\Su(C, \f{a})|$ for every $p$ and, $p+1$-sized ordered tuples of indices and all $\f{a} \in \{0,1\}^{p}\cup \{0,1\}^{p+1}$ (Lemma~\ref{lem:compute-SuCa}) using few queries.

\begin{defn}[flip-independent]\label{def:flip}
  A binary matrix $\f A$ with all distinct columns is called flip-independent if there exists a subset of rows that if complemented  (changing \texttt{0} to \texttt{1} and \texttt{1} to \texttt{0}) make all columns of $\f A$ linearly independent.

  Let $\mathcal{V}$ be a set of $\ell$ unknown vectors in $\bb{R}^n$, and $\f{A} \in \{0,1\}^{n \times \ell}$ be its support matrix. Let $\f{B}$ be the matrix obtained by deleting duplicate columns of $\f{A}$. The set $\mathcal{V}$ has flip-independent supports if $\bf{B}$ is flip-independent. 

  %for every column $A_i$, there exists a %unique tuple $(C, \bf a) \in [n]^t \times \{0,1\}^t $, where 
  %set $S \subset [n]$ of $t$-indices and a binary string $\bf{a} \in \{0,1\}^t$ such that $A_i|_S = \bf a$, and $A_j|_S \neq \bf a$ for all $j \neq i$.
\end{defn}

\paragraph{Support matrix is flip-independent:} 

Algorithm \ref{algo:s-lin-indep-supp-rec} uses the property that the support matrix $\f{A}$ is flip-independent in order to recover the supports of the unknown vectors uniquely. As a pre-processing step, we identify the set $\ca{U} \triangleq \cup_{i \in [\ell]}\s{supp}(\f{v}^i)$ that represents the union of support of the unknown vectors. Let us define $\ca{U}'\triangleq \ca{U} \cup \{t\}$ where $t$ is any index that does not belong to $\ca{U}$. This initial pre-processing step allows us to significantly reduce the computational complexity of this algorithm. Next, for each $\f{a}\in \{0,1\}^3$, Algorithm~\ref{algo:s-lin-indep-supp-rec} recovers  $|\Su((i_1, i_2, i_3), \f{a})|$ for every ordered tuple $(i_1, i_2, i_3) \in \ca{U}^3$. Using these recovered quantities, it is possible to construct the tensors 
\begin{align*}
    \ca{A}^{\ca{F}} = \sum_{i \in [\ell]} f(\ca{F},\s{supp}(\f{v}^i)) \otimes f(\ca{F},\s{supp}(\f{v}^i)) \otimes f(\ca{F},\s{supp}(\f{v}^i))  
\end{align*}
for every subset $\ca{F}\subseteq \ca{U}'$. 
 Since the matrix $\f{A}$ is flip-independent, we know that there exists at least one subset $\ca{F}^{\star}\subseteq \ca{U}'$ such that the vectors $\{f(\ca{F}^{\star},\s{supp}(\f{v}^i))\}_{i=1}^{\ell}$ are linearly independent. From Lemma \ref{lem:unique_cp}, we know that by a CP decomposition of $\ca{A}^{\ca{F}^{\star}}$, we can uniquely recover the vectors $\{f(\ca{F}^{\star},\s{supp}(\f{v}^i))\}_{i=1}^{\ell}$; since the set $\ca{F}^{\star}$ is known, we can recover all the vectors $\{\s{supp}(\f{v}^i))\}_{i=1}^{\ell}$ by flipping all indices corresponding to $\ca{F}^{\star}$. However, a remaining challenge is to correctly identify a set $\ca{F}^{\star}$. Interestingly, Jennrich's algorithm (see Algorithm \ref{algo:tensor} in Appendix \ref{sec:jenn}) throws an error if the tensor $\ca{A}^{\ca{F}}$ under consideration does not satisfy the uniqueness conditions for CP decomposition i.e. the underlying unknown vectors $\{f(\ca{F},\s{supp}(\f{v}^i))\}_{i=1}^{\ell}$ are not linearly independent. Therefore Algorithm \ref{algo:s-lin-indep-supp-rec} is guaranteed to uniquely recover the supports of the unknown vectors.

\begin{defn}[Kruskal rank]
The Kruskal rank of a matrix $\f A$ is defined as the maximum number $r$ such that any $r$ columns of $\f A$ are linearly independent. 
\end{defn}

\begin{defn}[$r$-Kruskal rank support]\label{def:kruskal}
Let $\mathcal{V}$ be a set of $\ell$ unknown vectors in $\bb{R}^n$, and $\f{A} \in \{0,1\}^{n \times \ell}$ be its support matrix. Let $\f{B}$ be the matrix obtained by deleting duplicate columns of $\f{A}$. The set $\mathcal{V}$ has $r$-Kruskal rank support if  $\bf{B}$ has Kruskal rank $r$.
\end{defn} 

\paragraph{Support matrix has Kruskal rank $r$:} Algorithm \ref{algo:s-indep-supp-rec}
partially generalizes the flip-independence property by constructing higher order tensors. Again, as a pre-processing step, we identify the set $\ca{U} \triangleq \cup_{i \in [\ell]}\s{supp}(\f{v}^i)$ that represents the union of support of the unknown vectors. Note that $|\ca{U}|\le \ell k$ since each unknown vector is $k$-sparse.
Since Jennrich's algorithm is not applicable for tensors of order more than 3, we will simply search over all $O((\ell k)^{\ell k})$ possibilities in order to compute the unique CP decomposition of an input tensor. Unfortunately though, if the sufficiency conditions (Lemma \ref{lem:unique_cp}) for unique CP decomposition is not met, there can be multiple solutions and we will not be able to detect the correct one. This is the reason why it is not possible to completely generalize Algorithm \ref{algo:s-lin-indep-supp-rec} by constructing multiple tensors of higher order.
 To circumvent this issue, Algorithm \ref{algo:s-indep-supp-rec} constructs only a single tensor $\ca A$ of rank $\ell$ 
and  order $w =\lceil \frac{2\ell-1}{r-1} \rceil$  by setting its $(i_1, \ldots, i_w)$-th entry to $|\Su((i_1, \ldots, i_w), \f{1}_w)|$ for every ordered tuple $(i_1, \ldots, i_w) \in [n]^w$.  By using Theorem \ref{lem:unique_cp}, the recovery of the supports of the unknown vectors via brute force CP decomposition of the constructed tensor is unique if the support matrix has Kruskal rank $r$.

All the above described algorithms require Lemma~\ref{lem:compute-SuCa} that for any $s >1$ computes $|\Su(C, \f{a})|$ for every $s$-sized ordered tuple of indices $C$, and any $\f{a} \in \{0,1\}^{s}$ using few label queries. 
The key technical ingredient in Lemma~\ref{lem:compute-SuCa} is to estimate $\s{nzcount}(\f x)$ - the number of unknown vectors that have a non-zero inner product with $\f{x}$. Note that even this simple task is non-trivial in the mixture model and more so with noisy label queries.

\subsection{Our Results}\label{sec:results}

%Recall that $\ca{V}$ is a set of $\ell$ unknown $k$-sparse vectors $\f{v}^1,\f{v}^2,\dots,\f{v}^{\ell}\in \bb{R}^n$. 

The MLC results below explicitly show the scaling  with the noise, whereas  all of the MLR query results below are valid with $$\s{SNR} = O(\ell^2\max_{i \in [\ell]}\left|\left|\f{v}^i\right|\right|_2^2/\delta^2).$$ 
  % This quantity will appear in the SNR for the MLR results.
%We utilize three distinct properties of the unknown set of vectors $\ca{V}$ to prove our query complexity guarantees. 
In our first result, we recover the support of the unknown vectors with small number of label queries provided the support matrix of $\ca{V}$ is $p$-identifiable.
\begin{theorem}\label{lem:t-iden-supp-rec}
Let $\mathcal{V}$ be a set of $\ell$ unknown vectors in $\R^n$ such that $\mathcal{V}$ is $p$-indentifiable. Then, Algorithm~\ref{algo:t-iden-supp-rec} recovers the support of all the unknown vectors in $\mathcal{V}$ 
with probability at least $1-O\left( 1/n \right)$ using either (1) $O\left( \frac{\ell^{3} (\ell k)^{p+2} \log(\ell k n) \log n }{ (1-2\eta)^2} \right)$ \MLC~queries or 
%$O((\sqrt{\s{SNR}}\cdot\delta)^{(p+5)} k^{p+2} \log(\s{SNR}\cdot\delta k n) \log n)$
(2) $O(\ell^{3} (\ell k)^{p+2} \log(\ell k n) \log n)$ 
\MLR~queries. % with $\s{SNR} = O(\ell^2\max_{i \in [\ell]}\left|\left|\f{v}^i\right|\right|_2^2/\delta^2)$.
%using $\tilde{O}()$ queries with probability at least $1-1/n$.
%with probability at least $1-1/n$ using $O(\ell^{3} (\ell k)^{s+1} \log(\ell k n) \log n / (1-2\eta)^2)$ \MLC~queries or $O(\ell^{3} (\ell k)^{s+1} \log(\ell k n) \log n)$ \MLR~queries.
\end{theorem}

In fact, all binary matrices with distinct columns are  $p$-identifiable for some  sufficiently large $p$.

\begin{theorem}\label{lem:suff-t}
Any $n \times \ell$, (with $n > \ell$) binary matrix with all distinct columns is  $p$-identifiable for some $p \le \log \ell$.
\end{theorem}
\begin{proof}%[Proof of Theorem \ref{lem:suff-t}]
Suppose $\f{A}$ is the said matrix. Since all the columns of $\f{A}$ are distinct, there must exist an index $i \in [n]$ which is not the same for all columns in $\f{A}$. We must have $\left|\s{occ}((i), a)\right| \le \ell/2$ for some $a \in \{0,1\}$. Subsequently, we consider the columns of $\f{A}$ indexed by the set $\s{occ}((i), a)$ and can repeat the same step. Evidently, there must exist an index $j \in [n]$ such that $\left|\s{occ}((i),\f{a})\right| \le \ell/4$ for some $\f{a} \in \{0,1\}^2$. Clearly, we can repeat this step at most $\log \ell$ times to find $C \subset [n]$ and $\f{a}\in \{0,1\}^{\le \log \ell}$ such that $\left|\s{occ}(C,\f{a})\right| = 1$ and therefore the column in $\s{occ}(C,\f{a})$ is $p$-identifiable. We denote the index of this column  as $\sigma(1)$ and form the sub-matrix $\f{A}^1$ by deleting the column. Again, $\f{A}^1$ has $\ell-1$ distinct columns and by repeating similar steps, $\f{A}^1$ has a column that is $\log(\ell-1)$ identifiable. More generally, $\f{A}^i$ formed by deleting the columns indexed in the set $\{\sigma(1),\sigma(2),\dots,\sigma(i-1)\}$, has a column that is $\log(\ell-i)$ identifiable with the index (in $\f{A}$) of the column having the unique sub-string (in $\f{A}^i$) denoted by $\sigma(i)$. Thus the lemma is proved.
\end{proof}

Thus, we have the following corollary characterizing the unconditional worst-case guarantees for support recovery:
\begin{coro}\label{coro:imp}
Let $\mathcal{V}$ be a set of $\ell$ unknown vectors in $\R^n$. 
Then, Algorithm~\ref{algo:t-iden-supp-rec} recovers the support of all the unknown vectors in $\mathcal{V}$ 
with probability at least $1-O\left(1/n\right)$ using either (1) $O\left( \frac{\ell^{3} (\ell k)^{\log \ell+2} \log(\ell k n) \log n}{ (1-2\eta)^2} \right)$ \MLC~queries or (2) $O(\ell^{3} (\ell k)^{\log \ell+2} \log(\ell k n) \log n)$ 
\MLR~queries. % with $\s{SNR} = O(\ell^2\max_{i \in [\ell]}\left|\left|\f{v}^i\right|\right|_2^2/\delta^2)$.
%using $\tilde{O}()$ queries with probability at least $1-1/n$.
%with probability at least $1-1/n$ using $O(\ell^{3} (\ell k)^{s+1} \log(\ell k n) \log n / (1-2\eta)^2)$ \MLC~queries or $O(\ell^{3} (\ell k)^{s+1} \log(\ell k n) \log n)$ \MLR~queries.
\end{coro}
\begin{proof}
The proof follows from the fact that any set $\ca{V}$ of $\ell$ unknown vectors in $\bb{R}^n$ must have $p$-identifiable supports for $p \le \log \ell$.
\end{proof}

Under some assumptions on the unknown support, e.g.flip-independence, we have better results. 
%Next, we characterize the query complexity if the support matrix of $\ca{V}$ is known to be flip-independent.

\begin{theorem}\label{lem:s-lin-indep-supp-rec}
Let $\mathcal{V}$ be a set of $\ell$ unknown vectors in $\R^n$ such that $\ca{V}$ is flip-independent.  
Then, Algorithm~\ref{algo:s-lin-indep-supp-rec} recovers the support of all the unknown vectors in $\mathcal{V}$ 
with probability at least $1-O\left(1/n\right)$ using either (1) $O\left( \frac{\ell^{3} (\ell k)^{4} \log(\ell k n) \log n}{ (1-2\eta)^2} \right)$ \MLC~queries or
%$O((\sqrt{\s{SNR}}\cdot\delta)^{7} k^{4} \log(\s{SNR}\cdot\delta k n) \log n)$ 
(2) $O(\ell^{3} (\ell k)^{4} \log(\ell k n) \log n)$ 
\MLR~queries. % with $\s{SNR} = O(\ell^2\max_{i \in [\ell]}\left|\left|\f{v}^i\right|\right|_2^2/\delta^2)$.
\end{theorem}

%{\color{red}This statement is not true as the sample complexity does not depend on $q$.}\gv{ verify if this statement is correct. }

%From Lemma \ref{lem:unique_cp}, we know that it is sufficient for the components of the tensor to have small Kruskal rank to ensure unique decomposition. We leverage this fact to obtain the following lemma.

We can also leverage the property of small Kruskal rank of the support matrix to show: % the result below. 
\begin{theorem}\label{lem:s-indep-supp-rec}
Let $\mathcal{V}$ be a set of $\ell$ unknown vectors in $\R^n$ that has $r$-Kruskal rank support with $r \ge 2$. Let $w = \lceil \frac{2\ell-1}{r-1} \rceil$. 
Then, Algorithm~\ref{algo:s-indep-supp-rec} recovers the support of all the unknown vectors in $\mathcal{V}$ 
with probability at least $1-O\left(1/n\right)$ using either (1) $O\left( \frac{\ell^{3} (\ell k)^{w+1} \log(\ell k n) \log n}{ (1-2\eta)^2} \right)$ \MLC~queries or 
%$O((\sqrt{\s{SNR}}\cdot\delta)^{(w+4)} k^{w+1} \log(\s{SNR}\cdot\delta k n) \log n)$
(2) $O(\ell^{3} (\ell k)^{w+1} \log(\ell k n) \log n)$ 
\MLR~queries. % with $\s{SNR} = O(\ell^2\max_{i \in [\ell]}\left|\left|\f{v}^i\right|\right|_2^2/\delta^2)$.
%using $\tilde{O}()$ queries with probability at least $1-1/n$.
%with probability at least $1-1/n$ using $O(\ell^{3} (\ell k)^{s+1} \log(\ell k n) \log n / (1-2\eta)^2)$ \MLC~queries or $O(\ell^{3} (\ell k)^{s+1} \log(\ell k n) \log n)$ \MLR~queries.
\end{theorem}

%\paragraph{Discussion on Matrix Properties:} Note that $p$-identifiability (Definition \ref{def:iden}) is  a generalization of the separability conditions outlined by \cite{gandikota2020recovery} for support recovery. This generalization allows us to recover the supports of all the unknown vectors  \textit{in the worst-case without any assumptions} (Corollary \ref{coro:imp}). The flip-independence (Definition \ref{def:flip}) and $r$-Kruskal rank support (Definition \ref{def:kruskal}) properties are used for the tensor-decomposition based support recovery algorithms and follow naturally from Lemma \ref{lem:unique_cp}. The flip-independence assumption is quite weak and we conjecture that all binary matrices with distinct columns are flip-independent. If this were true, then the query complexity for support recovery will have polynomial dependence on $\ell$ instead of $O(\ell^{\log \ell})$ dependence from Corollary \ref{coro:imp}. The $r$-Kruskal rank support condition generalizes linear independence conditions considered in previous mixture model studies such as \cite{yi2016solving}. Note that this condition is always satisfied by the support vectors for some $r\ge 2$.\textit{Essentially, we 1) provide algorithms for support recovery without any assumptions,  2) and  also provide significantly better guarantees under extremely mild assumptions that we conjecture to be always true.}

\paragraph{Discussion on Matrix Properties:} Note that $p$-identifiability (Definition \ref{def:iden}) is  a generalization of the separability conditions outlined by \cite{gandikota2020recovery} for support recovery. This generalization allows us to recover the supports of all the unknown vectors  \textit{in the worst-case without any assumptions} (Corollary \ref{coro:imp}).
The flip-independence (Definition \ref{def:flip}) and $r$-Kruskal rank support (Definition \ref{def:kruskal}) properties are used for the tensor-decomposition based support recovery algorithms and follow naturally from Lemma \ref{lem:unique_cp}. 
The flip-independence assumption is quite weak, however there do exist binary matrices that are not flip independent. For example 
\[M = \begin{bmatrix}
0&1&0&1\\
0&0&1&1\\
1&1&1&1\\
1&1&1&1
\end{bmatrix}\]
is not flip independent. 
%and we conjecture that all binary matrices with distinct columns are flip-independent. If this were true, then the query complexity for support recovery will have polynomial dependence on $\ell$ instead of $O(\ell^{\log \ell})$ dependence from Corollary \ref{coro:imp}. 
The $r$-Kruskal rank support condition generalizes linear independence conditions considered in previous mixture model studies such as \cite{yi2016solving}. Note that this condition is always satisfied by the support vectors for some $r\ge 2$.
\textit{Essentially, we 1) provide algorithms for support recovery without any assumptions,  2) and  also provide significantly better guarantees under extremely mild assumptions.}% that we conjecture to be always true.}

Although we have not optimized the run-time of our algorithms in this work, we report the relevant computational complexities below:

\begin{remark}[Computational Complexity]\label{rmk:comp}
Note that Algorithm \ref{algo:t-iden-supp-rec} has a computational complexity that is polynomial in the sparsity $k$, dimension $n$ and scales as $O(\ell^p)$ where $p\le \log \ell$. On the other hand Algorithms  \ref{algo:s-lin-indep-supp-rec}, \ref{algo:s-indep-supp-rec} has a computational complexity that scales exponentially with $k,\ell$ while remaining polynomial in the dimension $n$. For the special case when the support matrix is known to be full rank, Algorithm \ref{algo:s-indep-supp-rec} with $w=3$ is polynomial in all parameters (by using Algorithm \ref{algo:tensor} for the CP decomposition.) 
\end{remark}
% For the $\s{SNR}$ bound, we note that 
% %all queries made to the \MLR~oracle are binary vectors. Since
% the  \MLR-oracle queries are made only by Algorithm~\ref{algo:2}. From Lemma~\ref{lem:batchsize2}, we then get that
% $\s{SNR} = O(\ell^2\max_{i \in [\ell]}\left|\left|\f{v}^i\right|\right|_2^2/\delta^2)$.

%Finally, we show that even in arbitrary cases, a small number of queries are sufficient to recover the support of the unknown vectors. This proof follows from the 

%%%%%%%%%%%%%%%%%%%%%%%%%%%%%%%%%%%%%%%%%%%%%%%%%%%%%%%%%%%%%
%%%%%%%%%%%%%%%%%%%%%%%%%%%%%%%%%%%%%%%%%%%%%%%%%%%%%%%%%%%%%

%%%%%%%%%%%%%%%%%%%%%%%%%%%%%%%%%%%%%%%%%%%%%%%%%%%%%%%%%%%%%
%%%%%%%%%%%%%%%%%%%%%%%%%%%%%%%%%%%%%%%%%%%%%%%%%%%%%%%%%%%%%

\section{Detailed Proofs and Algorithms}\label{sec:detailed}

%For a set $\mathcal{V}$ of $\ell$ unknown vectors, let ${\bf A} \in \{0,1\}^{n \times \ell}$ denote their support matrix where each column vector $\f{A}_i \in \{0,1\}^n$ represents the support of the unknown vector $\f{v}^i$. 

Recall the definition of $\Su(C, \f{a})$ - the number of unknown vectors whose supports have $\f{a} \in \{0,1\}^{|C|}$ as a substring in coordinates $C \subset [n]$. 
First, we observe that for any set  $\ca{T} \subseteq \{0,1\}^s$, we can compute $|\Su(C, \f{a})|$ for all $O(n^s)$ subsets of $s$ indices $C \subset [n]$ and $\f a \in \ca T$ using few MLC or MLR queries. 
\begin{lemma}\label{lem:compute-SuCa}
%Let $\f{a} \in \{0,1\}^s$ be a binary vector. 
Let $\ca{T} \subseteq \{0,1\}^s$ be any set of binary vectors of length $s$. 
There exists an algorithm to compute $|\Su(C, \f{a})|$ for all $C \subset [n]$ of size $s$, and all $\f a \in \ca{T}$ with probability at least $1-1/n$ using either  $O(\ell^{3} (\ell k)^{s+1} \log(\ell k n) \log n / (1-2\eta)^2)$ \MLC~queries or 
%$O((\sqrt{\s{SNR}}\cdot\delta)^{(s+4)} k^{s+1} \log(\s{SNR}\cdot\delta k n) \log n)$
 $O(\ell^{3} (\ell k)^{s+1} \log(\ell k n) \log n)$ 
\MLR~queries. 
\end{lemma}
Lemma~\ref{lem:compute-SuCa} (proved in Section~\ref{sec:compute-SuCa}) is an integral and non-trivial component of the proofs of all our main Theorems. In each of the sub-sections below, we go through each of them.

\subsection{Recovery of $p$-identifiable support matrix }\label{subsec:p_identifiable}
In this section we present an algorithm for support recovery of all the $\ell$ unknown vectors $\ca{V}\equiv \{\f v^1, \ldots, \f v^\ell\}$  when $\ca{V}$ is $p$-identifiable. 
In particular, we show that if $\mathcal{V}$ is $p$-identifiable, then computing $|\Su(C, \f{a})|$ for every subset of $p$ and $p+1$ indices is sufficient to recover the supports. % of all the vectors.  
%We now present the support recovery algorithm for any set of $\ell$ vectors whose support is $t$-identifiable. 
%Put algorithm
\begin{proof}[Proof of Theorem~\ref{lem:t-iden-supp-rec}]

The proof follows from the observation that for any subset of $p$ indices $C \subset [n]$, index $j \in [n] \setminus C$ and $\f{a} \in \{0,1\}^p$, 
$|\Su(C, \f{a})| = |\Su(C\cup\{j\}, (\f{a}, 1))| + |\Su(C\cup\{j\}, (\f{a}, 0))|$. Therefore if one of the terms in the RHS is $0$ for all $j \in [n] \setminus C$, then all the vectors in $\Su(C, \f{a})$ share the same support. 

Also, if some two vectors $\f{u}, \f{v} \in \Su(C, \f{a})$ do not have the same support, then there will exist at least one index $j \in [n] \setminus C$ such that $ \f{u} \in \Su(C\cup\{j\}, (\f{a}, 1))|$ and $ \f{v} \in \Su(C\cup\{j\}, (\f{a}, 0))$ or the other way round, and therefore $|\Su(C\cup\{j\}, (\f{a}, 1))| \not \in \{0,|\Su(C, \f{a})|\} $. 
Algorithm~\ref{algo:t-iden-supp-rec} precisely checks for this condition. The existence of some vector $\f{v} \in \ca{V}$ ($p$-identifiable column), a subset of indices $C \subset [n]$ of size $p$, and a binary sub-string $\f{b}\in \{0,1\}^{\le p}$ follows from the fact that $\ca{V}$ is $p$-identifiable.
Let us denote the subset of unknown  vectors with distinct support in $\ca{V}$ by $\ca{V}^1$.

Once we recover the $p$-identifiable column of $\ca{V}^1$, we mark it as $\f{u}^{1}$ and remove it from the set (if there are multiple $p$-identifiable columns, we arbitrarily choose one of them). 
Subsequently,  we can modify the $\left|\Su{(\cdot)}\right|$ values for all $S \subseteq [n],|S|\in \{p,p+1\}$ and  $\f{t} \in \{0,1\}^{p} \cup \{0,1\}^{p+1} $ as follows:
\begin{align}\label{eq:comp_occ2}
    &\left|\Su^2(S, \f{t})\right| \triangleq \left|\Su(S, \f{t})\right| 
    - \left|\Su(C, \f{b})\right|\times\mathds{1}[ \supp{\f{u}^{1}}|_S = \f{t}].
\end{align}
Notice that, Equation~\ref{eq:comp_occ2} computes %we will also have obtained 
$
    \left|\Su^2(S, \f{t})\right| = \left|\{ \f{v}^i \in \mathcal{V}^2 \mid \supp{\f{v}^i}|_S = \f{t} \}\right|
$
where $\mathcal{V}^2$ is formed by deleting all copies of $\f{u}^{1}$ from $\ca{V}$. Since $\ca{V}^1$ is $p$-identifiable, there exists a $p$-identifiable column in $\ca{V}^1 \setminus \{\f{u}^{1}\}$ as well which we can recover. More generally for $q>2$, if $\f{u}^{q-1}$ is the $p$-identifiable column with the unique binary sub-string $\f{b}^{q-1}$ corresponding to the set of indices $C^{q-1}$, we will have for all $S \subseteq [n],|S|\in \{p,p+1\}$ and  $\f{t} \in \{0,1\}^{p} \cup \{0,1\}^{p+1} $
\begin{align*}
    \left|\Su^q(S, \f{t})\right| \triangleq \left|\Su^{q-1}(S, \f{t})\right| 
    - \left|\Su^{q-1}(C^{q-1}, \f{b}^{q-1})\right|\times\mathds{1}[ \supp{\f{u}^{q-1}}|_S = \f{t}] 
\end{align*}    
 implying 
 $\left|\Su^q(S, \f{t})\right| = \left|\{ \f{v}^i \in \mathcal{V}^q \mid \supp{\f{v}^i}|_S = \f{t} \}\right|
$
where $\mathcal{V}^q$ is formed deleting all copies of $\f{u}^{1},\f{u}^{2},\dots,\f{u}^{q-1}$ from $\ca{V}$. Applying these steps recursively and repeatedly using the property that $\ca{V}$ is $p$-identifiable, we can recover all the vectors present in $\ca{V}$. 

Algorithm~\ref{algo:t-iden-supp-rec} requires the values of $|\Su(C, \f a)|$, and $|\Su(\tilde{C}, \tilde{\f a})|$ for every $p$ and $p+1$ sized subset of indices $C, \tilde{C} \subset [n]$, and every $\f{a} \in \{0,1\}^p$, $\tilde{\f{a}} \in \{0,1\}^{p+1}$. 
Using Lemma~\ref{lem:compute-SuCa}, 
%for any fixed $\f{a} \in \{0,1\}^p$, 
we can compute 
%$|\Su(C, \f a)|$ for all $p$-sized subsets $C$ 
all these values using $O(\ell^{3} (\ell k)^{p+2} \log(\ell k n) \log n / (1-2\eta)^2)$ \MLC~queries or $O(\ell^{3} (\ell k)^{p+2} \log(\ell k n) \log n)$ \MLR~queries with probability at least $1 - O(n^{-1})$.
%Therefore for all the $2^p + 2^{p+1}$ binary vectors $\f{a}\in \{0,1\}^p$ and $\tilde{\f{a}} \in \{0,1\}^{p+1}$, the values of $|\Su(C, \f a)|$ and $|\Su(\tilde{C}, \tilde{\f a})|$ can be computed using $O(2^p \ell^{3} (\ell k)^{p+2} \log(\ell k n) \log n / (1-2\eta)^2)$ \MLC~queries or $O(2^p \ell^{3} (\ell k)^{p+2} \log(\ell k n) \log n)$ \MLR~queries with probability at least $1 - O(2^p/n)$ (by union bound).
\end{proof}

\subsection{Recovery of flip-independent support matrix}\label{subsec:flip-independent}

\begin{algorithm}[t]
\caption{\textsc{Recover $p$-identifiable Supports} \label{algo:t-iden-supp-rec}}
\begin{algorithmic}[1]
\REQUIRE $|\Su(C, \f a)|$ for every $C \subset [n]$, $|C| = t, \; t \in\{p, p+1\}$, and every $\f{a} \in \{0,1\}^{p} \cup \{0,1\}^{p+1}$.
%\REQUIRE $|\Su(\tilde{C}, \tilde{\f a})|$ for every $\tilde{C} \subset [n]$, $|\tilde{C}| = p+1$, and every $\tilde{\f{a}} \in \{0,1\}^{p+1}$.
\STATE Set $\text{count} = 1,i = 1 $.
\WHILE{$\text{count}\le \ell$}
    \IF{$|\Su(C, \f a)| = w$, and $|\Su(C\cup\{j\}, (\f{a}, 1))|\in \{0,w\}$ for all $j \in [n]\setminus C$}
        \STATE Set $\supp{\f{u}^{i}}|_C =  \f{a}$\\
        \STATE For every $j \in [n]\setminus C$, set $\supp{\f{u}^{i}}|_j = b$, where $|\Su(C\cup\{j\}, ({\f a}, b))| = w$.
        \STATE Set $\s{Multiplicity}^i = w$.
        \STATE For all $\f{t}\in \{0,1\}^{p} \cup \{0,1\}^{p+1},S \subseteq [n]$ such that $|S|\in \{p, p+1\}$, update
\begin{align*}
    \left|\Su(S, \f{t})\right| \leftarrow \left|\Su(S, \f{t})\right| 
    - \left|\Su(C, \f{a})\right|\times\mathds{1}[ \supp{\f{u}^{i}}|_S = \f{t}]
\end{align*}
        \STATE $\text{count} = \text{count} + w$.
        \STATE $i=i + 1$.
    \ENDIF
\ENDWHILE
\STATE Return $\s{Multiplicity}^j$ copies of $\supp{\f{u}^{j}}$ for all $j < i$.
\end{algorithmic}
\end{algorithm}

\begin{algorithm}[htbp!]
\caption{\textsc{Recover flip-independent Supports}\label{algo:s-lin-indep-supp-rec}}
\begin{algorithmic}[1]
%\STATE Let $w$ be smallest integer such that $w \cdot q \ge 2(\ell+1)$.
\REQUIRE $|\Su(C, \f{a})|$ for every $C \subset [n]$, such that $|C| = 3$, and all $\f{a} \in \{0,1\}^3$. $|\Su(i, 1)|$ for all $i\in [n]$.

\STATE Set $\ca{U} = \{i \in [n]:|\Su(i, 1)| \neq 0\}$ and  $\ca{U}' = \ca{U}\cup \{t\}$ where $t \in [n] \setminus \ca{U}$.

\FOR{each $\ca{F} \subset \ca{U}'$}
    \STATE Construct tensor $\ca{A}^{\ca F}$ as follows: 
    \FOR{every $(i_1, i_2, i_3) \in [n]^3$}
        \STATE Set $\ca{A}^{\ca F}_{(i_1, i_2, i_3)} = |\Su((i_1, i_2, i_3), (a_{i1}, a_{i2}, a_{i3}))|$,\\
        where $a_{ij} =0$ if $i_j \in \ca{F}$ and $1$ otherwise.
    \ENDFOR
    \IF{$\s{Jenerich}(\ca{A}^{\ca F})$ (Algorithm \ref{algo:tensor} with input $\ca{A}^{\ca{F}}$) succeeds:}
        \STATE Let $\ca{A}^{\ca F} = \sum_{i=1}^R \lambda_i \f{a}^i \otimes \f{a}^i \otimes \f{a}^i$ be the tensor decomposition of $\ca{A}$ such that $\f{a}^i \in \{0,1\}^n$.
        \STATE For all $i \in [R]$, modify $\f{a}^i$ by flipping  entries in $\ca{F}$.
        \STATE Return  $\lambda_i$ columns with modified $\f{a}^i$, $\forall i \in [R]$.
        \STATE {\bf break}
    \ENDIF
\ENDFOR
\end{algorithmic}
\end{algorithm}

In this section, we present an algorithm that recovers the support of all the $\ell$ unknown vectors in $\ca{V}$ provided $\ca{V}$ is flip-independent .

%Proof
\begin{proof}[Proof of Theorem~\ref{lem:s-lin-indep-supp-rec}]
The query complexity of the algorithm follows from Lemma~\ref{lem:compute-SuCa}. For any subset $C$ of $3$ indices, with probability  $1 - O(1/n)$, we can compute $|\Su{(C, \cdot)}|$ 
using $O(\ell^{3} (\ell k)^{4} \log(\ell k n) \log n / (1-2\eta)^2)$ \MLC~queries or 
%$O((\sqrt{\s{SNR}}\cdot\delta)^{(s+4)} k^{s+1} \log(\s{SNR}\cdot\delta k n) \log n)$
$O(\ell^{3} (\ell k)^{4} \log(\ell k n) \log n)$ 
\MLR~queries.

For every subset $\ca F \subseteq [n]$, we construct the tensor $\ca{A}^{\ca F}$ as follows:
$$\ca{A}^{\ca F}_{(i_1, i_2, i_3)} = |\Su((i_1, i_2, i_3), (a_{i1}, a_{i2}, a_{i3}))|,$$
for all $(i_1,i_2,i_3) \in [n]^3$
        where $a_{ij} =0$ if $i_j \in \ca{F}$ and $1$ otherwise.
We then run Jenerich's algorithm on each ${\ca A}^{\ca F}$. Observe that for any binary vector $\f b \in \{0,1\}^n$, the $(i_1, i_2, i_3)$-th entry of the rank-\texttt{1} tensor $\f b \otimes \f b \otimes \f b$ is $1$ if $b_{i_1}=b_{i_2}=b_{i_3}=1$, and $0$ otherwise. Therefore, the tensor ${\ca A}^{\ca F}$ can be decomposed as $\ca{A}^{\ca F} = \sum_{i=1}^R \lambda_i \f{a}^i \otimes \f{a}^i \otimes \f{a}^i$, where the  vectors $\f{a}^i \in \{0,1\}^n$, $i\in R$ are the support vectors of the unknown vectors that are flipped at indices in $\ca F$ with multiplicity $\lambda_i$.

Now if the support matrix of the unknown vectors is flip-independent, then there exists a subset of rows indexed by some $\ca{F}^\star \subseteq [n]$ such that flipping the entries of those rows results in a modified support matrix with all its distinct columns being linearly independent. 
Since the all zero rows of the support matrix $\f{A}$ are linearly independent (flipped or not), we can search for $\ca{F}^\star$ as a subset of $\ca{U}'$. Since,  $|\ca{U}'|\le \ell k+1$, this step improves the search space for $\ca{F}^\star$ from $O(2^n)$ to $O(2^{\ell k})$.

Therefore, Jenerich's algorithm on input ${\ca A}^{\ca F^\star}$ is guaranteed to succeed and returns the decomposition $\ca{A}^{\ca F^\star} = \sum_{i=1}^R \lambda_i \f{a}^i \otimes \f{a}^i \otimes \f{a}^i$ as the sum of $R$ rank-one tensors, where, $\f{a}^i \in \{0,1\}^n$, $i\in [R]$ are modified support vectors with multiplicity $\lambda_i$.
%Lemma~\ref{lem:unique_cp} guarantees the recovery of these vectors $\f{a}^i \in \{0,1\}^n$, $i\in [R]$ that correspond to the modified support vectors with multiplicity $\lambda_i$. %From the guarantees of Lemma \ref{lem:unique_cp}, we can recover the vectors $\f{a}^1,\dots,\f{a}^R$. % as they are linearly independent by using Jennrich's algorithm. 
Subsequently, we can again flip the entries of the recovered vectors indexed by $\ca{F}^\star$
to return the original support vectors. 
%To recover the supports we first construct the following order $r$ tensor: 
%\begin{align*}
%\ca{A}_{(i_1, \ldots, i_r)} &= |\Su((i_1, \ldots, i_r), \f{1}_r)|,
%\text{ for } (i_1, \ldots, i_r) \in [n]^r.
%\end{align*}
%where the entry $(i_1, \ldots, i_r) \in [n]^r$ is given by $|\Su((i_1, \ldots, i_r), (1, 1, \ldots, 1))|$. %The algorithm follows from the observation that the $[n]\times[n]\times \ldots \times[n]$ tensor formed by %$|\Su(C, \bf{1}_s)|$
%Observe that the tensor $\ca{A}$ can be written as the sum of $\ell$ rank one tensors:
%\begin{align}\label{eq:decomp}
%\ca{A} = \sum_{i=1}^\ell \underbrace{\supp{\f{v^i}} \otimes %\ldots \otimes \supp{\f{v^i}}}_{r\text{-times}}.
%\end{align}
%Since the support matrix $A$ of $\ca{V}$ is $q$-independent, for any $r$ such that $r \cdot q \ge 2(\ell+1)$, the decomposition of Equation~\ref{eq:decomp} is unique (Theorem~\ref{thm:tensor-decomp}).
%Therefore, the matrix returned by the Algorithm contains the supports of individual unknown vectors as its columns. 
%it follows from Theorem~\ref{thm:tensor-decomp} there exists a unique decomposition of $\ca{A}$. 
%that can recover $A$, and hence the supports of all the unknown vectors up to permutations.
%Algorithm~\ref{algo:s-indep-supp-rec} needs to know the values of $|\Su(C, \f{1}_q)|$ for every $C \subset [n]$, such that $|C| = q$. Using Lemma~\ref{lem:compute-SuCa}, these can be computed using $O(\ell^{3} (\ell k)^{q+1} \log(\ell k n) \log n / (1-2\eta)^2)$ \MLC~queries or $O(\ell^{3} (\ell k)^{q+1} \log(\ell k n) \log n)$ \MLR~queries with probability at least $1 - O(1/n)$.
\end{proof}
%%%%%%%%%%%%%%%%%%%%%%%%%%%%%%%%%%%%%%
%%%%%%%%%%%%%%%%%%%%%%%%%%%%%%%%%%%%%%

\subsection{Recovery of $r$-Kruskal rank supports}\label{subsec:r-kruskal}
\begin{algorithm}[htbp!]
\caption{\textsc{Recover $r$-Kruskal rank Supports}\label{algo:s-indep-supp-rec}}
\begin{algorithmic}[1]
\STATE Let $w$ be smallest integer such that $w \cdot (r-1) \ge 2\ell-1$.
\REQUIRE $|\Su(C, \f{1}_w)|$ for every $C \subset [n]$ with $|C| = w$. $|\Su((i),1)|$ for all $i\in [n]$.
\STATE Set $\ca{U} \triangleq \{i \in [n]:|\Su((i),1)|\neq 0\}$.
\STATE Construct tensor $\ca{A}$ as follows: 
\FOR{every $(i_1, \ldots, i_w) \in [n]^w$}
    \STATE Set $\ca{A}_{(i_1, \ldots, i_w)} = |\Su((i_1, \ldots, i_w), \f{1}_w
    )|$.
\ENDFOR
\FOR{every $(\f{b}^1,\f{b}^2 \ldots, \f{b}^{\ell}) \in \{0,1\}^n$ satisfying $\s{supp}(\f{b}^i)\subseteq \ca{U}$}
    \IF{$\ca{A}=\sum_{i=1}^{\ell}\f{b}^i\otimes\f{b}^i\dots\otimes \f{b}^i$ ($w$ times)} 
    \STATE Set $(\f{b}^1,\f{b}^2 \ldots, \f{b}^{\ell})$ to be the CP decomposition of $\ca{A}$ and Break
    \ENDIF
\ENDFOR
\STATE Return CP decomposition of $\ca{A}$
\end{algorithmic}
\end{algorithm}

In this section, we present an algorithm that recovers the support of all the $\ell$ unknown vectors provided they have $r$-Kruskal rank supports. 
Recall that for any set of $w$ indices $C \subset [n]$, $\Su(C, \f{1}_w)$ denotes the set of unknown vectors that are supported on all indices in $C$.

%Proof
\begin{proof}[Proof of Theorem~\ref{lem:s-indep-supp-rec}]
To recover the supports we first construct the following order $w$ tensor: 
$
\ca{A}_{(i_1, \ldots, i_w)} = |\Su((i_1, \ldots, i_w), \f{1}_w)|,
\text{ for } (i_1, \ldots, i_w) \in [n]^w.
$
%where the entry $(i_1, \ldots, i_r) \in [n]^r$ is given by $|\Su((i_1, \ldots, i_r), (1, 1, \ldots, 1))|$. %The algorithm follows from the observation that the $[n]\times[n]\times \ldots \times[n]$ tensor formed by %$|\Su(C, \bf{1}_s)|$
Observe that the tensor $\ca{A}$ can be written as the sum of $\ell'$ ($\ell'<\ell$) rank one tensors 
\begin{align}\label{eq:decomp}
\ca{A} = \sum_{i=1}^{\ell'}\lambda^i \underbrace{\supp{\f{v}^i} \otimes \ldots \otimes \supp{\f{v}^i}}_{w\text{-times}}.
\end{align}
where $\f{v}^1,\f{v}^2,\dots,\f{v}^{\ell'}$ are the unknown vectors with distinct supports in $\ca{V}$ with $\lambda^i$ being the multiplicity of $\supp{\f{v}^i}$.
Since the support matrix $\f{A}$ of $\ca{V}$ has $r$-Kruskal rank, for any $w$ such that $w \cdot (r-1) \ge 2\ell-1$, the decomposition of Eq.~\eqref{eq:decomp} is unique (Lemma~\ref{lem:unique_cp}).
 Notice that by a pre-processing step, we compute $\ca{U}\triangleq \{i \in [n]:|\Su((i),1)|\neq 0\}$ to be the union of the supports of the unknown vectors. Since we know that the underlying vectors of the tensor that we construct are binary, we can simply search exhaustively over all the possibilities ($O((\ell k)^{\ell k})$ of them (Steps 7-10) to find the unique CP decomposition of the tensor $\ca{A}$. For the special case when $w=3$, Jennrich's algorithm (Algorithm \ref{algo:tensor}) can be used to efficiently compute the unique CP decomposition of the tensor $\ca{A}$.

%{\color{red}This paragraph needs more explanation. The matrix has $r$-Kruskal rank supports. What is $q$ here?}
%it follows from Theorem~\ref{thm:tensor-decomp} there exists a unique decomposition of $\ca{A}$. 
%that can recover $A$, and hence the supports of all the unknown vectors up to permutations.

Algorithm~\ref{algo:s-indep-supp-rec} needs to know the values of $|\Su(C, \f{1}_w)|$ for every $C \subset [n]$, such that $|C| = w$. Using Lemma~\ref{lem:compute-SuCa}, these can be computed using $O(\ell^{3} (\ell k)^{w+1} \log(\ell k n) \log n / (1-2\eta)^2)$ \MLC~queries or 
%$O((\sqrt{\s{SNR}}\cdot\delta)^{(w+4)} k^{w+1} \log(\s{SNR}\cdot\delta k n) \log n)$
$O(\ell^{3} (\ell k)^{w+1} \log(\ell k n) \log n)$ 
\MLR~queries with probability at least $1 - O(1/n)$.
\end{proof}

\section{Computing $\Su(C, \bf{a})$ }\label{sec:compute-SuCa}
In this section, we provide the proof of Lemma~\ref{lem:compute-SuCa} that follows from the correctness and performance guarantees of the subroutines that for any $s < n$, compute $|\bigcup_{i \in \ca{S}} \Su((i), 1)|$ for every subset of indices $\ca{S}$ of size $s$.

Let $s < n$, then using queries constructed from $\s{CFF}$s of appropriate parameters %along with  Algorithm~\ref{algo:Si} as a subroutine, 
we compute $|\bigcup_{i \in \ca{S}} \Su((i), 1)|$ for all subsets $\ca{S} \subset [n]$ of size $s$. % using oracle queries. 
\begin{lemma}\label{lem:uff}
For any $1< s < n$, there exists an algorithm to compute $\left|\bigcup_{i \in \ca{S}} \Su((i),1)\right|$ for all $\ca{S} \subseteq [n],\left|\ca{S}\right|=s$ with probability at least $1 - O(n^{-2})$ using $O(\ell^3(\ell k)^{s+1} \log(\ell k n) \log n / (1-2\eta)^2)$ \MLC~queries or %$O((\sqrt{\s{SNR}}\cdot\delta)^{(s+4)} k^{s+1} \log(\s{SNR}\cdot\delta k n) \log n)$
$O(\ell^3(\ell k)^{s+1} \log(\ell k n) \log n)$ 
\MLR~queries.
\end{lemma}
The proof of Lemma~\ref{lem:uff} follows from the guarantees of Algorithm~\ref{algo:SUnion} provided in Section~\ref{sec:SUnion}.

For the special case of $s=1$, we use  queries given by a $\s{RUFF}$ of appropriate parameters to compute $|\Su((i), 1)|$ for all $i \in [n]$ using Algorithm~\ref{algo:Si} in Section~\ref{sec:Si}. 
\begin{lemma}{\label{lem:ruff}}
There exists an algorithm to compute $|\Su((i), 1)|~\forall~i \in [n]$ with probability at least $1 - O(n^{-2})$ using $O(\ell^4 k^2 \log(\ell k n) \log n / (1-2\eta)^2 )$ \MLC queries, or 
%$O(\s{SNR}^2  \delta^4 k^2 \log(\s{SNR}\cdot \delta  k n) \log n )$
$O(\ell^4 k^2 \log(\ell k n) \log n )$ 
\MLR queries. 
\end{lemma}
Both the above mentioned algorithms crucially use a subroutine that counts the number of unknown vectors in $\ca{V}$ that have a non-zero inner product with a given query vector $\f{x}$. For any $\f{x} \in \R^n$, define
$
\s{nzcount}(\f{x}) := \sum_{i=1}^\ell \mathds{1}[ \langle \f{v}^i, \f{x} \rangle \neq 0].
$
The algorithm to estimate $\s{nzcount}$ for \MLC~is similar to that of \cite{gandikota2020recovery}. However, in this work we consider the  general setting of noisy \MLC ~queries, i.e., the responses to the queries can be erroneous in sign with some small probability $\eta$. Therefore  we include the proof in Section~\ref{sec:nzcount}.

\begin{lemma}\label{lem:batchsize}
There exists an algorithm that computes $\s{nzcount}(\f{x})$ for any vector $\f{x} \in \bb{R}^n$, with probability at least $1-2e^{-T(1-2\eta)^2/2\ell^2}$ using $2T$ \MLC~queries.
    %For any vector $\f{x} \in \bb{R}^n$, Algorithm~\ref{algo:1}
   %%the estimator $\hat{\s{nz}}(\f{x})$ 
   %with a batchsize of $T$ estimates $\s{nzcount}(\f{x})$ with probability at %least $1-2e^{-T(1-2\eta)^2/2\ell^2}$ using only $2T$ oracle queries.
\end{lemma}

The problem of estimating $\s{nzcount}(\f x)$ in the mixed linear regression model is slightly more challenging due to the presence of additive noise. 
Note that one can scale the queries with some large positive constant to minimize the effect of the additive noise. However, we also aim to minimize the $\s{SNR}$, and hence need more sophisticated techniques to estimate $\s{nzcount}(\f x)$. We restrict our attention to only binary vectors $\f{x}$ to estimate $\s{nzcount}$ in \MLR~model which is sufficient for support recovery.

\begin{lemma}\label{lem:batchsize2}
There exists an algorithm to compute $\s{nzcount}(\f{x})$ for any vector $\f{x} \in \{0,1\}^n$, with probability at least $1-2\exp(-T/36\pi\ell^2)$ using only $T$ \MLR  queries. Moreover, $\s{SNR} = O(\ell^2\max_{i \in [\ell]}\left|\left|\f{v}^i\right|\right|_2^2/\delta^2)$.
  % For any vector $\f{x} \in \{0,1\}^n$, Algorithm~\ref{algo:2} with %$a=\sigma/2$, $\gamma=2\sqrt{2}\ell \sigma$, and batch size $T$ estimates %$\s{nzcount}(\f{x})$ with probability at least %$1-2\exp\Big(-\frac{T}{36\pi\ell^2}\Big)$ using $T$ oracle queries.
  % %the estimator $\hat{\s{z}}(\f{x})$ with a batch-size of $T$, $a=\sigma/2$, %$\gamma=2\sqrt{2}\ell \sigma$ correctly estimates $\s{z}(\f{x})$ with %probability at least $1-2\exp\Big(-\frac{T}{36\pi\ell^2}\Big)$.
\end{lemma}
% We can write the query complexity in terms of the $\s{SNR}$ to prove the following corollary.
% \begin{coro}
% Assume $\left|\left|\f{v}^i\right|\right|_2=1$ for all $i\in [\ell]$. There exists an algorithm to compute $\s{nzcount}(\f{x})$ for any vector $\f{x} \in \{0,1\}^n$, with probability at least $1-2\exp\Big(-\Omega\Big(\frac{T}{\s{SNR}\cdot \delta^2}\Big)\Big)$ using only $T$ \MLR~oracle queries.
% \end{coro}
% \begin{proof}
%  Substituting $\ell^2 = O(\s{SNR}\cdot\delta^2)$ into the proof of Lemma \ref{lem:batchsize2}, we get the proof of the corollary. 
% \end{proof}
%In the next step, we characterize the number of samples sufficient to recover $\left|\bigcap_{i \in \ca{F}} \ca{S}(i)\right|$ for all $\ca{F} \subseteq [n]$ such that $\left|\ca{F} \right|\le s$ for any $s \in [n]$. 
%\begin{lemma}\label{lem:intersection}
%For any $s \in [n]$, given $\left|\bigcup_{i \in \ca{F}} \ca{S}(i)\right|$ for all %$\ca{F} \subseteq [n]$ such that $\left|\ca{F}\right| \le s$, we can 
%compute $\left|\bigcap_{i \in \ca{F}} \ca{T}(i)\right|$ for all $\ca{F} \subseteq [n]$ and for all $\ca{T}(i) \in \{\ca{S}(i),\ca{S}(i)^c\}$ such that $\left|\ca{F}\right| \le s$. 
%\end{lemma}
\begin{proof}[Proof of Lemma~\ref{lem:compute-SuCa}]% BEGIN PROOF OF LEMMA lem:compute-SuCa
Using Algorithm~\ref{algo:SUnion} $(s-1)$ times and Algorithm~\ref{algo:Si}, we can compute $\left|\bigcup_{i \in \ca{S}} \Su((i), 1)\right|$ for all $\ca{S} \subseteq [n]$ such that $\left|\ca{S}\right|\le s$. 

From Lemma~\ref{lem:uff}, we know that each call to Algorithm~\ref{algo:SUnion} with any $t \le s$ uses $O(\ell^{3} (\ell k)^{t+1} \log(\ell k n) \log n / (1-2\eta)^2)$ \MLC~queries, and each succeeds with probability at least $1 - O(1/n^2)$. 
Therefore, taking a union bound over all $t < s$, we can compute $\left|\bigcup_{i \in \ca{S}} \Su((i), 1)\right|$ for all $\ca{S} \subseteq [n]$, $|\ca{S}| \le s$ using $O(\ell^{3} (\ell k)^{s+1} \log(\ell k n) \log n / (1-2\eta)^2)$ \MLC~queries with probability $1-O(1/n)$. 
Alternately, we can compute the quantities using $O(\ell^{3} (\ell k)^{s+1} \log(\ell k n) \log n)$ \MLR~queries with probability $1-O(1/n)$.

%Using Lemma \ref{lem:uff} and taking a union bound over all $t \le s$, we can compute the aforementioned quantities using $O(\ell^{s+3} k^s \log(\ell k n) \log n)$ queries with probability $1-O(1/n)$. 

We now show using by induction on $s$ that the quantities $\left\{ \left|\bigcup_{i \in \ca{S}} \Su((i), 1)\right| ~\forall~\ca{S} \subseteq [n], |\ca{S}| \le s \right\}$
are sufficient to compute $|\Su(C, \f{a})|$ for all subsets $C$ of indices of size at most $s$, and any binary vector $\f{a} \in \{0,1\}^{\le s}$.

%since we can infer both $|\Su(i, 1)|$ and $|\Su(i, 0)| = \ell - |\Su(i, 1)|$  from $|\Su(i, 1)|$ computed using Algorithm \ref{algo:Si} for all $i \in [n]$.

%Subsequently, note that for any $t$ sets $\ca{A}_1,\ca{A}_2,\dots,\ca{A}_t$, we must have

%We will prove this lemma by using induction. \\
\textit{Base case ($t=1$):} The base case follows  since we can infer both $|\Su((i), 1)|$ and $|\Su((i), 0)| = \ell - |\Su((i), 1)|$  from $|\Su((i), 1)|$ computed using Algorithm \ref{algo:Si}  $\forall i \in [n]$. 

%We have already computed $\left|\ca{S}(i)\right|$ in Steps $12-15$ of  Algorithm \ref{algo:Si} for all $i \in [n]$. Clearly, we can also compute $\left|\ca{S}_i^c\right| = n- \left|\ca{S}(i)\right|$ for all $i \in [n]$. This resolves the base case for $t=1$. \\
% \textit{Inductive Hypothesis:} Let us assume the induction hypothesis that for any $\f{a} \in \{0,1\}^{s-1}$, we can compute $|\Su(C, \f{a})|$ for all subsets $C \subset [n]$, $|C| = s-1$ using $\left\{ \left|\bigcup_{i \in \ca{S}} \Su((i), 1)\right| ~\forall~\ca{S} \subseteq [n], |\ca{S}| < s \right\}$.

\textit{Inductive Step:}
Let us assume that the statement is true for $r < s$ i.e. we can compute $|\s{occ}(\ca{C},\f{a})|$ for all subsets $\ca{C}$ satisfying $|\ca{C}|\le r$ and any binary vector $\f{a}\in \{0,1\}^{\le r}$ from the quantities $\left\{ \left|\bigcup_{i \in \ca{S}} \Su((i), 1)\right| ~\forall~\ca{S} \subseteq [n], |\ca{S}| \le r \right\}$ provided as input. Now, we claim that the statement is true for $r+1$. For simplicity of notation we will denote by $\ca{S}_i \triangleq \s{occ}({i},1)$ the set of unknown vectors which have a $1$ in the $i^{\s{th}}$ entry. Note that we can also rewrite $\s{occ}(\ca{C},\f{a})$ for any set $\ca{C}\subseteq [n], \f{a}\in \{0,1\}^{|\ca{C}|}$ as \begin{align*}
    \s{occ}(\ca{C},\f{a}) = \bigcap_{j \in \ca{C}'}\ca{S}_j \bigcap_{j \in \ca{C}\setminus \ca{C}' }\ca{S}_j^c
\end{align*}
where $\ca{C}'\subseteq \ca{C}$ corresponds to the indices in $\ca{C}$ for which the entries in $\f{a}$ is \texttt{1}. Fix any set $i_1,i_2,\dots,i_{r+1} \in [n]$. Then we can compute $\left|\bigcap_{b=1}^{r+1} \ca{S}_{i_b}\right|$ using the following equation:
\begin{align*}
    (-1)^{r+3}\left|\bigcap_{b=1}^{r+1} \ca{S}_{i_b}\right| 
    = \sum_{u=1}^{r}(-1)^{u+1} \sum_{\substack{j_1,j_2,\dots,j_u \in \{i_1,i_2,\dots,i_{r+1}\}\\ j_1 < j_2 <\dots<j_u  }} \left|\bigcap_{b=1}^{u} \ca{S}_{j_b}\right|-\left|\bigcup_{b=1}^{r+1} \ca{S}_{i_b}\right|.
\end{align*}

Finally for any proper subset $\ca{Y} \subset \{i_1,i_2,\dots,i_{r+1}\}$, we can compute $\left|\bigcap_{ i_b \not \in \ca{Y}} \ca{S}_{i_b} \bigcap_{i_b \in \ca{Y}} \ca{S}_{i_b}^{c}\right|$ using the following set of equations:

\begin{align*}
    \left|\bigcap_{i_b \not \in \ca{Y}} \ca{S}_{i_b} \bigcap_{i_b \in \ca{Y}} \ca{S}_{i_b}^{c}\right| &=     \left|\bigcap_{i_b \not \in \ca{Y}} \ca{S}_{i_b} \bigcap \Big(\bigcup_{i_b \in \ca{Y}} \ca{S}_{i_b} \Big)^{c} \right| \\
    &= \left|\bigcap_{i_b \not \in \ca{Y}} \ca{S}_{i_b} \right| - \left| \bigcap_{i_b \not \in \ca{Y}} \ca{S}_{i_b} \bigcap \Big(\bigcup_{i_b \in \ca{Y}} \ca{S}_{i_b} \Big) \right| \\
    &=\left|\bigcap_{i_b \not \in \ca{Y}} \ca{S}_{i_b} \right| - \left|\bigcup_{i_b \in \ca{Y}} \Big(\bigcap_{i_b \not \in \ca{Y}} \ca{S}_{i_b} \bigcap \ca{S}_{i_b} \Big) \right|. 
\end{align*}
The first term is already pre-computed and the second term is again a union of intersection of sets. For any $i_b \in \ca{Y}$, let us define $\ca{Q}_{i_b} := \bigcap_{i_b \not \in \ca{Y}} \ca{S}_{i_b} \bigcap \ca{S}_{i_b}$. Therefore we have 
\begin{align*}
    \left|\bigcup_{i_b \in \ca{Y}} \ca{Q}_{i_b}\right|= \sum_{u=1}^{\left|\ca{Y}\right|}(-1)^{u+1} \sum_{\substack{j_1,j_2,\dots,j_u \in \ca{Y}\\ j_1 < j_2 <\dots<j_u  }} \left|\bigcap_{b=1}^{u} \ca{Q}_{j_b}\right|.
\end{align*}
We can compute $\left|\bigcup_{i_b \in \ca{Y}} \ca{Q}_{i_b}\right|$ because the quantities on the right hand side of the equation have already been pre-computed (using our induction hypothesis). Therefore, the lemma is proved.

% We now show how to compute $|\Su(C, \f{a})|$ for all subsets $C \subset [n]$, $|C| = s$, and for any fixed $\f{a} \in \{0,1\}^{s}$ using $\left\{ \left|\bigcup_{i \in \ca{S}} \Su((i), 1)\right| ~\forall~\ca{S} \subseteq [n], |\ca{S}| \le s \right\}$.

% The proof follows from the observation that for any fixed set of $s-1$ indices  $\tilde{C} \subset [n]$, index $i_s \in [n] \setminus \tilde{C}$, $\tilde{\f{a}} \in \{0,1\}^{s-1}$, and $b \in \{0,1\}$, we have 
% \begin{align*}
%     &|\Su(\tilde{C} \cup \{i_s\}, (\tilde{\f{a}}, b))| = |\Su(\tilde{C}, \tilde{\f{a}}) \cap \Su(i_s, b)| \\
%     &= |\Su(\tilde{C}, \tilde{\f{a}})| + |\Su(i_s, b)| - |\Su(\tilde{C}, \tilde{\f{a}}) \cup \Su(i_s, b)|.
% \end{align*}
% The fact that we can compute $|\Su(\tilde{C}, \tilde{\f{a}})|$ follows from the induction hypothesis.
Therefore, for any subset $\ca T \subset \{0,1\}^s$, we can compute $\left\{|\Su(C, \f a)| \mid~\forall \f a \in \ca T, C \subset [n], |C| = s \right\}$ by computing $\left\{ \left|\bigcup_{i \in \ca{S}} \Su((i), 1)\right| ~\forall~\ca{S} \subseteq [n], |\ca{S}| \le s \right\}$ just once.
\end{proof}

\section{Missing Proofs and Algorithms in computing $\Su(C, \bf{a})$}\label{sec:Sunion}

\subsection{Compute $\left|\bigcup_{i \in \ca{S}} \Su((i),1)\right|$ using Algorithm~\ref{algo:SUnion}. }\label{sec:SUnion}

%In the next step, we characterize the number of samples sufficient to recover $\left|\bigcup_{i \in \ca{F}} \ca{S}(i)\right|$ for all $\ca{F} \subseteq [n]$ such that $\left|\ca{F} \right|=s$ for some $s \in [n]$. 

In this section we present an algorithm to compute $\left|\bigcup_{i \in \ca{S}} \Su((i),1)\right|$, for every $\ca{S} \subseteq [n]$ of size  $\left|\ca{S}\right|=s$, using $|\Su((i),1)|$ computed in the Section~\ref{sec:Si}.

We will need an $(s,\ell k) $-CFF for this purpose. 
Let $\ca{G} \equiv \{\ca{H}_1, \ca{H}_2, \dots, \ca{H}_n \}$ be the required $(s,\ell k) $-CFF of size $n$ over alphabet $m = O((\ell k)^{s+1}\log n)$. 
%For the mixtures of linear regression, we construct the matrix $B \in \bb{R}^{m \times n}$ from the $(s,\ell k) $-CFF $\ca{G}$ by setting  
%For every $(i,j) \in [m] \times [n]$, set 
%$\f{B}_{i,j} = 1$ if only if $i \in H_j$, and $0$ otherwise for every %$(i,j) \in [m] \times [n]$.
We construct a set of $\ell + 1$ matrices $\ca{B} = \{\f{B^{(1)}}, \ldots,  \f{B^{(\ell+1)}}\}$ 
where, each $\f{B^{(w)}} \in \bb{R}^{m \times n}, w \in [\ell+1]$, is obtained from the $(s,\ell k) $-CFF $\ca{G}$. The construction of these matrices varies slightly for the model in question. 

For the mixture of linear classifiers, we construct the sequence of matrices as follows: For every $(i,j) \in [m] \times [n]$, set $\f{B}^{(w)}_{i,j}$ to be a random number sampled uniformly from $[0,1]$ if $i \in H_j$, and $0$ otherwise. We remark that the choice of uniform distribution in $[0,1]$ is arbitrary, and any continuous distribution works. 
Since every $\f{B^{(w)}}$ is generated identically, they have the exact same support, though the non-zero entries are different. Also, by definition, the support of the columns of every $\f{B^{(w)}}$ corresponds to the sets in $\ca{G}$. 

For the mixture of linear regressions, we avoid the scaling of non-zero entries by a uniform scalar. We set $\f{B}^{(w)}_{i,j}$ to be $1$ if $i \in H_j$, and $0$ otherwise. Note that in this case each $\f{B}^{(w)}$ is identical. We see that the scaling by uniform scalar is not necessary for the mixtures of linear regressions since the procedure to compute $\s{nzcount}$ in this model (see  Algorithm~\ref{algo:2}) scales the query vectors by a Gaussian scalar which is sufficient for our purposes.

Let $\ca{U} := \cup_{i \in [\ell]} \s{supp}(\f{v}^{i})$ denote the union of supports of all the unknown vectors. Since each unknown vector is $k$-sparse, it follows that $| \ca{U} | \le \ell k$. From the properties of $(s,\ell k) $-CFF, we know that for any tuple of $s$ indices $(i_1, i_2,\dots, i_s) \subset \ca{U}$, the set 
$(\bigcap_{t=1}^{s}\ca{H}_{i_i})  \setminus \bigcup_{q \in \ca{U} \setminus \{ i_1,i_2,\dots, i_s\}} \ca{H}_{q}$
is non-empty. 
This implies that for every $w \in [\ell+1]$, there exists at least one row of $\f{B^{(w)}}$ that  has a non-zero entry in the $i_1^{\s{th}},i_2^{\s{th}},\dots,i_s^{\s{th}}$ index, and $0$ in all other indices $p \in U \setminus \{ i_1,i_2,\dots,i_s\}$. In Algorithm~\ref{algo:SUnion} we use these rows as queries to estimate their $\s{nzcount}$. In Lemma~\ref{lem:uff}, we show that this estimated quantity is exactly $|\bigcup_{j=1}^{s} \Su((i),1)|$ for that particular tuple $(i_1, i_2,\dots, i_s) \subset \ca{U}$.

% COMPUTE |S(i) cap S(j) |
\begin{algorithm}[h!]
\caption{\textsc{Recover Union-} $\left|\bigcup_{i \in \ca{S}} \Su((i),1)\right|$ for all $\ca{S} \subseteq [n], \left|\ca{S}\right|=s, s \ge 2.$ \label{algo:SUnion}}
\begin{algorithmic}[1]
\REQUIRE $|\Su((i),1)|$  for every $i \in [n]$. $s \ge 2$. 
\REQUIRE Construct $\f{B} \in \bb{R}^{m \times n}$ from $(s,\ell k) $-CFF of size $n$ over alphabet $m=c_3 (\ell k)^{s+1}\log n$.
\STATE Let $\ca{U} := \{i \in [n] \mid |\Su((i),1) | >0\}$
\STATE Let batchsize $T_C = 10\ell^2 \log(nm)/(1-2\eta)^2$, \\
$T_R = 10 \cdot (36\pi)\ell^2 \log(nm)$.
\FOR{every $p \in [m]$}
\STATE Let $\s{count}(p) :=  \max_{ w \in [\ell+1]}\{\s{nzcount}(\f{B^{(w)}}[p])\}$ \\(obtained using Algorithm~\ref{algo:1} with batchsize $T_C$ for MLC, or Algorithm~\ref{algo:2} with batchsize $T_R$ for MLR).
\ENDFOR
\FOR{every set $\ca{S} \subseteq [n]$ with $\left|\ca{S}\right|=s$}
		\STATE Let $p \in [m]$ such that $\f{B_{p,t}} \neq 0$ for all $t \in \ca{S}$,  and $\f{B_{p,t'}} = 0$ for all $q \in \ca{U} \setminus \ca{S}$.
		\STATE Set $\left|\bigcup_{i \in \ca{S}} \Su((i),1)  \right| = \s{count}(p)$.
\ENDFOR
\end{algorithmic}
\end{algorithm}

\begin{proof}[Proof of Lemma~\ref{lem:uff}]
Computing each $\s{count}$ (see Algorithm \ref{algo:SUnion}, line 8) requires $O(T \ell)$ queries, where $T=T_C$ for \MLC, and $T=T_R$ for \MLR. 
Therefore, the total number of queries made by Algorithm~\ref{algo:SUnion} is at most 
\begin{align*}
    O(m T_C \ell) &= O\big(\frac{(\ell k)^{s+1} \ell^{3} \log(\ell k n) \log n}{(1-2\eta)^2}\big)\\
    O(m T_R \ell) &= O((\ell k)^{s+1} \ell^{3} \log(\ell k n) \log n)
\end{align*} 
for $m = O((\ell k)^{s+1}\log n)$, $T_C = O(\ell^2 \log(n m)/ (1-2\eta)^2)$, and $T_R =O(\ell^2 \log(n m))$.
Also, observe that each $\s{nzcount}$ is estimated correctly with probability at least $1 - O\left(1/ \ell mn^2 \right)$. Therefore from union bound it follows that all the $(\ell+1)m$ estimations of $\s{count}$ are correct with probability at least $1 - O\left(1/n^2\right)$.

Recall that the set $\ca{U}$ denotes the union of supports of all the unknown vectors. 
This set is equivalent to $\{i \in [n] \mid |\Su((i),1) | > 0 \}$. 
% First, note that if $| \ca{S}(i) | = 0$, there are no unknown vectors supported on the $i^{th}$ index. 
% Therefore, $|\ca{S}(i) \cup \ca{S}(j)| = |\ca{S}(j)|$. 
% Also, if $i=j$, then the computation of $|\ca{S}(i) \cup \ca{S}(j)|$ is trivial. 

% We now focus on the only non-trivial case when $(i, j) \in U \times U$ and $i \neq j$. 

Since for every $w \in [\ell+1]$, the support of the columns of $\f{B^{(w)}}$ are the indicators of sets in $\ca{G}$, the $(s,\ell k)$-CFF property implies that there exists at least one row (say, with index $p \in [m]$) of every $\f{B^{(w)}}$ which has a non-zero entry in the $i_1^{\s{th}},i_2^{\s{th}},\dots,i_s^{\s{th}}$ index, and $0$ in all other indices $q \in U \setminus \{i_1,i_2,\dots,i_s\}$, i.e., 
\begin{align*}
&\f{B^{(w)}_{p,t}} \neq 0 \; \mbox{ for all  } t\in \{i_1,i_2,\dots,i_s\} \mbox{, and } \\
&\f{B^{(w)}_{p,t'}} = 0 \mbox{ for all  } t' \in \ca{U} \setminus \{ i_1,i_2,\dots,i_s\}.
\end{align*}
To prove the correctness of the algorithm, we need to show the following:
\begin{align*}
\left|\bigcup_{p \in \{i_1,i_2,\dots,i_s\}} \Su(p,1) \right|
%\max_{w \in [\ell+1]} \{ \hat{\s{nz}}(\f{B^{(w)}}[p]) \} \\
&= \max_{w \in [\ell+1]} \{ \s{nzcount}(\f{B^{(w)}}[p]) \}
\end{align*}
First observe that using the row $\f{B^{(w)}}[p]$ as query will produce non-zero value for only those unknown vectors $\f{v} \in \bigcup_{p \in \{i_1,i_2,\dots,i_s\}} \Su(p,1)$. This establishes the fact that $|\bigcup_{p \in \{i_1,i_2,\dots,i_s\}} \Su(p,1)| \ge \s{nzcount}(\f{B^{(w)}}[p])$. 

To show the other side of the inequality, consider the set of $(\ell+1)$ $s$-dimensional vectors obtained by the restriction of rows $\f{B^{(w)}}[p]$ to the coordinates $(i_1, i_2,\dots,i_s)$, 
\[
\{ ( \f{B^{(w)}_{p,i_1}},  \f{B^{(w)}_{p,i_2}},\dots, \f{B^{(w)}_{p,i_s}} ) \mid w \in [\ell+1] \}.
\] 

For MLC, these entries are picked uniformly at random from $[0,1]$, they hence are pairwise linearly independent. For MLR, since the $\s{nzcount}$ scales the non-zero entries of the query vector $\f{B^{(w)}}[p]$ by a Gaussian, the pairwise linear independence still holds. 
Therefore, each $\f{v} \in  \bigcup_{p \in \{i_1,i_2,\dots,i_s\}} \Su(p,1)$ can have $\langle \f{B^{(w)}}[p], \f{v} \rangle = 0$ for at most $1$ of the $w$ queries. So by pigeonhole principle, at least one of the query vectors $\f{B^{(w)}}[p]$ will have $\langle \f{B^{(w)}}[p], \f{v} \rangle \neq 0$ for all $\f{v} \in  \bigcup_{p \in \{i_1,i_2,\dots,i_s\}} \Su(p,1)$. Hence, $|\bigcup_{p \in \{i_1,i_2,\dots,i_s\}} \Su(p,1)| \le \max_w \{ \s{nzcount}(\f{B^{(w)}}[p]) \}$. 

\end{proof}

\subsection{Computing $|\Su((i),1)|$}\label{sec:Si}
%using Algorithm~\ref{algo:Si}.} 
In this section, we show how to compute $|\Su(i, 1)|$ for every index $i \in [n]$. 

Let $\ca{F} =  \{\ca{H}_1, \ca{H}_2, \dots, \ca{H}_n \}$ be a $(d,\ell k, 0.5)$-$\s{RUFF}$ of size $n$ over alphabet $[m]$.   
Construct the binary matrix $\f{A} \in \{0,1\}^{m \times n}$ from $\ca{F}$, as $\f{A}_{i,j} = 1$ if and only if $i \in \ca{H}_j$. 
Each column $j \in [n]$ of $\f{A}$ is essentially the indicator vector of the set $\ca{H}_j$. 
%From Lemma \ref{lem:ruffexist}, we know the existence of a $(d,\ell k,0.5) - \s{RUFF}$ exists with $m = c_1\ell ^2 k^2 \log n$ and $d = c_2 \ell k \log n$ for some absolute constants $c_1,c_2>0$. 

We use the rows of matrix $\f{A}$ as query vectors to compute $|\Su((i), 1)|$ for each $i \in [n]$. For each such query vector $\f{x}$, we compute the $\s{nzcount}(\f{x})$ using Algorithm~\ref{algo:1} with batchsize $T_C$ for \MLC, and Algorithm~\ref{algo:2} with batchsize $T_R$ for \MLR.  % = O(\ell^2 \log \ell k n)$. The large value of $T$ 
We choose $T_C$ and $T_R$ to be sufficiently large to ensure that $\s{nzcount}$ is correct for all the queries with very high probability.

%Roughly describe algorithm. 
For every $h \in \{0, \ldots, \ell \}$, let $\f{b}^h \in \{0,1\}^m$ be the indicator of the queries that have $\s{nzcount}$ at least $h$. We show in Lemma~\ref{lem:ruff} that the set of columns of $\f{A}$ that have large intersection with $\f{b}^h$, exactly correspond to the indices $i \in [n]$ that satisfy $|\Su((i),1)| \ge h$. This allows us to recover $|\Su((i),1)|$ exactly for each $i \in [n]$. 
%The correctness and the query complexity follows from Lemma~\ref{lem:ruff}. 
% COMPUTE |S(i)| 
\begin{algorithm}[h!]
\caption{\textsc{Compute--}$|\Su((i),1)|$ \label{algo:Si}}
\begin{algorithmic}[1]
\REQUIRE Construct binary matrix $\f{A} \in \{ 0,1\}^{m \times n}$ from 
$(d,\ell k,0.5)- \s{RUFF}$ of size $n$ over alphabet $[m]$, with $m=c_1\ell^2k^2\log n$ and $d=c_2\ell k\log n$. 
\STATE Initialize $\f{b}^0, \f{b}^1,\f{b}^2,\dots,\f{b}^\ell$ to all zero vectors of dimension $m$.  
\STATE Let batchsize $T_C = 4 \ell^2 \log mn / (1-2\eta)^2$ for \MLC, and $T_R = 4\cdot (36\pi)\cdot \ell^2 \log mn$ for \MLR. 
\FOR{$i=1,\dots,m$}
\STATE Set $w := \s{nzcount}(\f{A}[i])$\\ (obtained using Algorithm~\ref{algo:1} with batchsize $T_C$ for \MLC, and Algorithm~\ref{algo:2} with batchsize $T_R$ for \MLR.)
\FOR{$h=0,1,\dots,w$}
\STATE Set $\f{b}^h_i=1$.
\ENDFOR
\ENDFOR
\FOR{$h=0,1,\dots,\ell$}
\STATE Set $\ca{C}_h=\{i \in [n] \mid |\s{supp}(\f{b}^h)\cap \s{supp}(\f{A}_i)|\ge 0.5 d \}$.
\ENDFOR
\FOR{$i=1,2,\dots,n$}
\STATE Set $|\Su((i),1)|=h$ if $i \in \{\ca{C}_{h}\setminus \ca{C}_{h+1}\}$ for some $h \in \{0,1,\dots, \ell-1\}$.
\STATE Set $|\Su((i),1)|=\ell$ if $i \in \ca{C}_{\ell}$
\ENDFOR
\end{algorithmic}
\end{algorithm}

\begin{proof}[Proof of Lemma~\ref{lem:ruff}]
Since $\f{A}$ has $m = O(\ell^2k^2\log n)$ distinct rows, and each row is queried $T_{C} = O( \ell^2 \log(mn)/ (1-2\eta)^2)$ times for MLC     and $T_R = O( \ell^2 \log(mn))$ times for \MLR, the total query complexity of Algorithm~\ref{algo:Si} is $O(\ell^4 k^2 \log(\ell k n) \log n /(1-2\eta)^2 )$ for \MLC, and $O(\ell^4 k^2 \log(\ell k n) \log n )$ for \MLR. 

To prove the correctness, we first see that the $\s{nzcount}$ for each query is estimated correctly using Algorithm~\ref{algo:1} with overwhelmingly high probability. 
From Lemma \ref{lem:batchsize} with $T_{C} = 4\ell^2 \log(mn) / (1-2\eta)^2$, it follows that each $\s{nzcount}$ is estimated correctly with probability at least $1 - \frac{1}{mn^2}$. 
Therefore, by taking a union bound over all rows of $\f{A}$, we estimate all the counts accurately with probability at least $1-\frac{1}{n^2}$ for \MLC. The bounds follow similarly for \MLR from Lemma~\ref{lem:batchsize2} with $T_R = 4\cdot (36\pi)\cdot \ell^2 \log mn$.

We now show, using the properties of $\s{RUFF}$, that $|\s{supp}(\f{b}^h)\cap \s{supp}(\f{A}_i)|\ge 0.5 d$ if and only if $|\Su((i),1)| \ge h$, for any $0\le h \le \ell$. 
Let $i \in [n]$ be an index such that $|\Su((i),1)| \ge h$, i.e., there exist at least $h$ unknown vectors that have a non-zero entry in their $i^{th}$ coordinate. Also, let $U := \cup_{i \in [\ell]} \s{supp}(\f{v}^{i})$ denote the union of supports of all the unknown vectors. Since each unknown vector is $k$-sparse, it follows that $| U | \le \ell k$. 
To show that $|\s{supp}(\f{b}^h)\cap \s{supp}(\f{A}_i)|\ge 0.5 d$, consider the set of rows of $\f{A}$ indexed by $W := \{\s{supp}(\f{A}_i) \setminus \cup_{j \in U \setminus \{i\} } \s{supp}(\f{A}_j)\}$. 
Since $\f{A}$ is a $(d, \ell k, 0.5) - \s{RUFF}$, we know that $|W| \ge 0.5 d$. We now show that $\f{b}^h_t = 1$ for every $t \in W$. 
This follows from the observation that for $t \in W$, and each unknown vector $\f{v} \in \Su((i),1)$, the query 
 $\langle \f{A}[t],\f{v} \rangle = \f{v}_i \neq 0$.  
Since $|\Su((i),1)| \ge h$, we conclude that $\s{nzcount}(\f{A}[t]) \ge h$, and therefore, $\f{b}^h_t = 1$. 

To prove the converse, consider an index $i \in [n]$ such that $|\Su((i),1)| < h$. Using a similar argument as above, we now show that $|\s{supp}(\f{b}^h)\cap \s{supp}(\f{A}_i)| < 0.5 d$.  Consider the set of rows of $\f{A}$ indexed by $W := \{\s{supp}(\f{A}_i) \setminus \cup_{j \in U \setminus \{i\} } \s{supp}(\f{A}_j)\}$. Now observe that for each $t \in W$, and any unknown vector $\f{v} \notin \Su((i),1)$, $\langle \f{A}[t],\f{v} \rangle = 0$.  Therefore $\s{nzcount}(\f{A}[t]) \le |\Su((i),1)| < h$, and $\f{b}^h_t =0$ for all $t \in W$. Since $|W| \ge 0.5 d$, it follows that $|\s{supp}(\f{b}^h)\cap \s{supp}(\f{A}_i)| < 0.5 d$. For any $0 \le h \le \ell$, Algorithm~\ref{algo:Si}. therefore correctly identifies the set of indices $i \in [n]$ such that $|\Su((i),1)| \ge h$. In particular, the set $C_h: = \{i \in [n] \mid |\Su((i),1)| \ge h\}$. 
Therefore, the set $\ca{C}_{h} \setminus \ca{C}_{h+1}$ is exactly the set of indices $i \in [n]$ such that $|\Su((i),1)| = h$. 
\end{proof}

%\begin{proof}[Proof of Lemma~\ref{lem:ruff2}]
%The correctness of Algorithm~\ref{algo:Si} for MLR follows from the proof of Lemma~\ref{lem:ruff}. We now show that query complexity bounds for the lemma. 

%Since $\f{A}$ has $m = O(\ell^2k^2\log n)$ distinct rows, and each row is queried $T_{R} = O( \ell^2 \log(mn))$ times, the total query complexity of Algorithm~\ref{algo:Si} for MLR is $O(\ell^4 k^2 \log(\ell k n) \log n)$. 

%To prove the correctness, we first see that the $\s{nzcount}$ for each query is estimated correctly using Algorithm~\ref{algo:1} with overwhelmingly high probability. 
%Also, from Lemma \ref{lem:batchsize2} with $T_{R} = 4(36 \pi)\ell^2 \log(mn)$, it follows that each $\s{nzcount}$ is estimated correctly with probability at least $1 - \frac{1}{mn^2}$. 
%Therefore, by taking a union bound over all rows of $\f{A}$, we estimate all the counts accurately with probability at least $1-\frac{1}{n^2}$. 
%\end{proof}

%%%%%%%%%%%%%%%%%%%%%%%%%%%%%%%%%%%%%%%%%%%%%%%%%%%%%%%%%%%%%%%%%%
\subsection{Estimating $\s{nzcount}$}\label{sec:nzcount}
The main subroutine used to compute both $|\Su((i), 1)|$ and $|\cup_j \Su((j), 1)|$ is to estimate $\s{nzcount}(\f x)$ - the number of unknown vectors that have a non-zero inner product with $\f x \in \R^n$.  We now provide algorithms to estimate $\s{nzcount}(\f x)$ using very few queries in both the models considered in this work. 

\subsubsection{Estimating $\s{nzcount}$ for Mixture of Linear Classifiers}\label{sec:mlc-nzcount}

%we will use the oracle queries to estimate $\s{nzcount}(\f{a})$ - the number of unknown vectors that have a non-zero response with $\f{a}$, i.e.,  
%\begin{align*}
%\s{nzcount}(\f{a}) 
%&= \lvert \{ \f{v^j} \mid \s{sign}(\langle \f{a}, \f{v^j} \rangle) \neq 0, j \in [\ell] \} %\rvert.
%\end{align*}
%We now show how to estimate $\s{nzcount}$ using small number of oracle queries. 
%The algorithm is morally similar to that of \cite{gandikota2020recovery}. 
Algorithm~\ref{algo:1} empirically estimates $\s{nzcount}$ by repeatedly querying with the same vectors $\f x$ and its negation $-\f{x}$. Let $T$ denote the number of times a fixed query vector $\f x$ is repeatedly queried. We refer to this quantity as the \emph{batchsize}. We now show that Algorithm~\ref{algo:1} estimates $\s{nzcount}$ with overwhelmingly high probability.

\begin{algorithm}[h!]
\caption{\textsc{\textsc{Query}}$(\f{x},T)$\label{algo:1}}
\begin{algorithmic}[1]
\REQUIRE Query access to $\ca{O}$.
%An Oracle $\ca{O}$ which when queried with a vector $\f{v} \in \bb{R}^{n}$ returns $\s{sign}(\langle \f{v},\f{\beta} \rangle)~\in~\{ -1, 0, 1\}$ where $\f{\beta}$ is sampled uniformly from the set of unknown vectors $\{\f{\beta}^1,\f{\beta}^2,\dots,\f{\beta}^\ell\}$.
\FOR{$i=1,2,\dots,T$}
\STATE Query with vector $\f{x}$ and obtain response $y^{i} \in \{-1, +1\}$.
\STATE Query with vector $-\f{x}$ and obtain response $z^{i} \in \{-1, +1\}$.
\ENDFOR
\STATE Let $\hat{\s{z}}:=   \s{round}\Big(\frac{\ell\sum_{i=1}^{T}y_i+z_i}{2T(1-2\eta)}\Big)$.
\STATE Return $\hat{\s{nz}} = \ell - \hat{\s{z}}$.
\end{algorithmic}
\end{algorithm}

\begin{proof}[Proof of Lemma~\ref{lem:batchsize}]
Let us define the quantity $\s{zcount}(\f x)$ to denote the number of unknown vectors that have a zero inner product with $\f{x}$. Note it is sufficient to estimate this quantity accurately since $\s{nzcount}(\f x) = \ell - \s{zcount}(\f x)$ can be inferred directly from it. 
The algorithm is based on the following observation that for any fixed query vector $\f{x}$,
\begin{align*}
&\underset{\f{v} \sim_U \ca{V}}{\bb{E}}[\ca{O}(\f{x})]\\
    &=\Big(\underset{\f{v} \sim_U \ca{V}}{\bb{E}}[\mathds{1}[\langle \f{x},\f{v} \rangle \ge 0]] - \underset{\f{v} \sim_U \ca{V}}{\bb{E}}[\mathds{1}[\langle \f{x},\f{v} \rangle < 0]] \Big)(1-2\eta)\\
    %&= \Big(\Pr_{\f{v} \sim_U \ca{V}}\Big( \langle \f{x},\f{v} \rangle \ge 0 \Big)-\Pr_{\f{v} \sim_U \ca{V}}\Big( \langle \f{x},\f{v} \rangle < 0 \Big) \Big)(1-2\eta)\\
    &= \Big(\frac{1}{\ell} \cdot \sum_{i=1}^{\ell} \mathds{1}[\langle \f{x},\f{v}^i \rangle \ge 0]-\frac{1}{\ell} \cdot \sum_{i=1}^{\ell} \mathds{1}[\langle \f{x},\f{v}^i \rangle < 0]\Big)(1-2\eta).
\end{align*}
Note that since 
\begin{align*}
    &\mathds{1}[\langle \f{x},\f{v}^i \rangle \ge 0]-\mathds{1}[\langle \f{x},\f{v}^i \rangle < 0] \\
    &= \mathds{1}[\langle \f{x},-\f{v}^i \rangle \ge 0]-\mathds{1}[\langle \f{x},-\f{v}^i \rangle < 0] \quad \text{if} \quad \langle \f{x},\f{v}^i \rangle=0
\end{align*}
and 
\begin{align*}
&\mathds{1}[\langle \f{x},\f{v}^i \rangle \ge 0]-\mathds{1}[\langle \f{x},\f{v}^i \rangle < 0]\\
&= \mathds{1}[\langle \f{x},-\f{v}^i \rangle < 0]-\mathds{1}[\langle \f{x},-\f{v}^i \rangle \ge 0] \quad \text{if} \quad \langle \f{x},\f{v}^i \rangle \neq 0.
\end{align*}
Therefore, we must have 
\begin{align*}
    \frac{\bb{E}_{\f{v} \sim_U \ca{V}}[\ca{O}(\f{x})+\ca{O}(-\f{x})]}{2(1-2\eta)} 
    & = \frac{1}{\ell} \cdot \sum_{i=1}^{\ell} \mathds{1}[\langle \f{x},\f{v}^i \rangle = 0] \\
    &=\frac{1}{\ell} \cdot \s{zcount}(\f x)
\end{align*}

The algorithm therefore empirically estimates $\s{zcount}(\f x)$ using repeated queries with vectors $\f x$ and $-\f x$.
Let us denote the the $T$ responses from $\ca{O}$ by $y_1,y_2,\dots,y_T$ and $z_1,z_2,\dots,z_T$ corresponding to the query vectors $\f{x}$ and $-\f{x}$ respectively.

From the observations stated above, it then follows that the quantity $U= \frac{\ell}{ (1-2\eta)}\frac{\sum_i y_i+z_i}{2T}$ is an unbiased estimate for $\s{zcount}(\f x)$, i.e. 
$\bb{E}U = \s{zcount}(\f x)$. Algorithm~\ref{algo:1} therefore makes a mistake in estimating $ \s{zcount}(\f x)$ (i.e., $\hat{\s{z}} \neq \s{zcount}(\f x)$) only if 
\begin{align*}
|U- \bb{E}U| \ge \frac{1-2\eta}{2\ell}.
\end{align*}

%We show that this happens with very 
Since the responses to the queries are independent, using Chernoff bounds  \cite{boucheron2013concentration} it then follows that the algorithm makes an erroneous estimate of $\s{zcount}(\f x)$ with very low probability.
%we get an upper bound on the probability that Algorithm~\ref{algo:1} makes a mistake as
\begin{align*}
\Pr \Big(|U - \bb{E}U| \ge \frac{1-2\eta}{2\ell} \Big) \le 2e^{-\frac{T(1-2\eta)^{2}}{2\ell^2}}. 
\end{align*}

\end{proof}

%%%%%%%%%%%%%%%%%%%%%%%%%%%%%%%%%%%%%%%%%%%%%%%%%%%%%%%%%%%%%%%%%%
\subsubsection{Estimating $\s{nzcount}$ for Mixed Linear Regressions}\label{sec:mlr-nzcount}
%The problem of estimating $\s{nzcount}(\f x)$ in this model is slightly more challenging due to the presence of real noise. 
%Note that one can scale the queries with some large positive constant to minimize the effect of the additive noise. However, we also aim to minimize the $\s{SNR}$, and hence need more sophisticated techniques to estimate $\s{nzcount}(\f x)$. 

%The algorithm used to estimate this quantity is in essence similar to Algorithm~\ref{algo:1}. However, to minimize $\s{SNR}$, 

We restrict our attention to only binary queries in this section which is sufficient for support recovery. Algorithm~\ref{algo:2} queries repeatedly with a carefully crafted transformation $\f{\Tr_\gamma}(\f{x})$ of the input vector $\f{x}$, and counts the number of responses that lie within a fixed range $[-a, a]$. This estimates  count the number of unknown vectors that have a zero inner product with $\f{x}$, and thereby estimates $\s{nzcount}(\f{x})$.

For any binary vector $\f{x} \in \{0,1\}^n$, define as follows:
%. In order to do that we
%design the  
$\f{\Tr_\gamma}:\{0,1\}^n \rightarrow \bb{R}^n$ %to generate a random query vector in the following way:
\begin{align*}
    \f{\Tr_\gamma}(\f{x})_i = 
    \begin{cases}
     0 \; \text{if } \; \f{x}_i = 0 \\
     \ca{N}(0,\gamma^2) \; \text{if } \; \f{x}_i \neq 0.
    \end{cases}
\end{align*}

For any $a, \sigma \in \bb{R}$, let us also define %Let us define 
\begin{align*}
    \phi_1(a, \sigma) &:= \Pr_{W \sim \ca{N}(0,\sigma^2)}(W \in [-a,a]) \quad \text{and}\\ 
    \quad \phi_2(a, \sigma, \gamma) &:= \Pr_{W \sim \ca{N}(0,\sigma^2+\gamma^2)}(W \in [-a,a]). 
\end{align*}

From standard Gaussian concentration bounds, we know that 
\begin{eqnarray}\label{eq:gaussian}
    &\phi_1(a,\sigma) = \s{erf}\Big(\frac{a}{\sqrt{2}\sigma}\Big) \ge \frac{\sqrt{2}}{\sqrt{\pi}}\Big(\frac{a}{\sigma}-\frac{a^3}{6\sigma^3}\Big). \\
    &\phi_2(a,\sigma,\gamma) = \s{erf}\Big(\frac{a}{\sqrt{2(\sigma^2+\gamma^2)}}\Big) \le a\sqrt{\frac{2}{\pi(\sigma^2+\gamma^2)}}.
\end{eqnarray}

\begin{algorithm}[h!]
\caption{\textsc{\textsc{Query}}$(\f{x} \in \{0,1\}^{n},T,a,\gamma)$\label{algo:2}}
\begin{algorithmic}[1]
\REQUIRE Query access to $\ca{O}$ and known $\sigma,\ell$.
%An Oracle $\ca{O}$ which when queried with a vector $\f{v} \in \bb{R}^{n}$ returns $\s{sign}(\langle \f{v},\f{\beta} \rangle)~\in~\{ -1, 0, 1\}$ where $\f{\beta}$ is sampled uniformly from the set of unknown vectors $\{\f{\beta}^1,\f{\beta}^2,\dots,\f{\beta}^\ell\}$.
\FOR{$i=1,2,\dots,T$}
\STATE Query with vector $\f{\Tr_\gamma}(\f{x})$ and obtain response $y_{i} \in \bb{R}$.
\ENDFOR
\STATE Let $\hat{\s{z}} = \s{round}\Big(\frac{\ell \sum_{i=1}^T \mathds{1}\left[y_i \in [-a,a]\right]}{T\phi_1(a,\sigma)}\Big).$

\STATE Return $\hat{\s{nz}}=\ell-\hat{\s{z}}(\f{x})$.
\end{algorithmic}
\end{algorithm}

\begin{proof}[Proof of Lemma~\ref{lem:batchsize2}]

Similar to the proof of Lemma~\ref{lem:batchsize} define $\s{zcount}(\f x)$ denote the number of unknown vectors that have a zero inner product with $\f{x}$. We show that Algorithm~\ref{algo:2} estimates this quantity accurately, and hence $\s{nzcount}(\f x) = \ell - \s{zcount}(\f x)$ can be inferred from it.

%As before, our first step is to characterize the batch-size to correctly estimate $\s{nzcount}(\f{x})$ 

For the set of $T$ responses $y_1, \ldots, y_T$ obtained from $\ca{O}$, define $U :=\frac{\sum_i \mathds{1}\left[y^i \in [-a,a]\right]}{T}$. Then,  
\begin{align}\label{eq:1}
    \underset{{\cal V}, \f{\Tr_\gamma}, Z}{\bb{E}}[U] = \underset{{\cal V}, \f{\Tr_\gamma}, Z}{\Pr}\Big(\langle \f{\Tr_\gamma}(\f{x}),\f{v} \rangle +Z \in [-a,a]\Big).
\end{align}

%$\bb{E}_{{\cal V}, \f{g_\gamma}, Z} U = \Pr_{{\cal V}, \f{g_\gamma}, Z} \Big(\langle \theta(\f{x}),\f{v} \rangle +Z \in [-a,a]\Big)$

Note that for any $a \in \bb{R}$ and $\f{x} \in \{0,1\}^n$, we have 
\begin{align*}
    &\underset{{\cal V}, \f{\Tr_\gamma}, Z}{\Pr}\Big(\langle \f{\Tr_\gamma(\f{x})},\f{v} \rangle +Z \in [-a,a]\Big) \\
    &= \frac{1}{\ell}\Bigg(\sum_{i:\langle \f{x}, \f{v}^i \rangle =0} \underset{ \f{\Tr_\gamma}, Z}{\Pr}\Big(\langle \f{\Tr_\gamma(\f{x})},\f{v}^i \rangle +Z \in [-a,a]\Big) \\
    &+\sum_{i:\langle \f{x}, \f{v}^i \rangle \neq 0} \underset{\f{\Tr_\gamma}, Z}{\Pr}\Big(\langle \f{\Tr_\gamma(\f{x})},\f{v}^i \rangle +Z \in [-a,a]\Big)\Bigg)
\end{align*}

Observe that if $\langle \f{x},\f{v}^i \rangle =0$, then $\langle \f{\Tr_\gamma(\f{x})}, \f{v}^i \rangle+Z \sim \ca{N}(0,\sigma^2)$, and if $\langle \f{x},\f{v}^i \rangle \neq 0$, then 
%Note that if $\langle \f{x}, \f{v}^i \rangle+Z \sim \ca{N}(0,\sigma^2)$, then $\langle \theta(\f{x}),\f{v}^i \rangle =0$ as well otherwise 
$\langle \f{\Tr_\gamma(\f{x})},\f{v}^i \rangle \sim \ca{N}(0,\gamma^2 \left|\left| \f{x} \odot \f{v}^i\right|\right|_2^2+\sigma^2)$, where $\f{u} \odot \f{v}$ denotes the entry-wise product of $\f{u},\f{v}$. 
It then follows that
\begin{align}\label{eq:2}
     &\frac{\s{zcount}(\f{x})}{\ell} \cdot \phi_1(a, \sigma) \le \underset{{\cal V}, \f{\Tr_\gamma}, Z}{\Pr}\Big(\langle \f{\Tr_\gamma}(\f{x}),\f{v} \rangle +Z \in [-a,a]\Big) \nonumber \\ 
     &\le  \frac{\s{zcount}(\f{x})}{\ell} \cdot \phi_1(a, \sigma) +\phi_2(a, \sigma, \gamma\delta). 
\end{align}

Setting the parameters $a=\sigma/2$ and $\gamma = 2\sqrt{2\ell} \sigma/\delta$, from Equation~\ref{eq:gaussian}, we get that 
\begin{align*}
    \phi_1(a,\sigma) \ge \frac{23\sqrt{2}}{48\sqrt{\pi}} \quad \text{and} \quad \phi_2(a,\sigma,\gamma\delta) \le \frac{\sqrt{2}}{4\ell \sqrt{\pi}}.
\end{align*}
and therefore, $4\ell \phi_2(a,\sigma,\gamma\delta)\le \phi_1(a,\sigma).$

Combining this observation with Equation~\ref{eq:1} and Equation~\ref{eq:2}, we then get that 
\begin{align}\label{eq:3}
     &\frac{\s{zcount}(\f{x})}{\ell} \cdot \phi_1(a, \sigma) \le \underset{{\cal V}, \f{\Tr_\gamma}, Z}{\bb{E}}[U] \nonumber \\ 
     &\le  \frac{\s{zcount}(\f{x})}{\ell} \cdot \phi_1(a, \sigma) + \frac{1}{4\ell}\cdot \phi_1(a, \sigma). 
\end{align}

From Equation~\ref{eq:3}, we observe that if $|U - \bb{E}[U]| \le \frac{1}{4\ell}\cdot \phi_1(a, \sigma)$, then $\s{zcount}(\f{x}) -\frac14 \le \frac{\ell U}{\phi_1(a, \sigma)} \le \s{zcount}(\f{x}) + \frac12$. Since $\s{zcount}(\f{x})$ is integral, it follows that if $|U - \bb{E}[U]| \le \frac{1}{4\ell}\cdot \phi_1(a, \sigma)$, the estimate $\s{\hat{z}} = \s{round}\big(\frac{\ell U}{\phi_1(a, \sigma)}\big)$ computed in Algorithm~\ref{algo:2} will correctly estimate  $\s{zcount}(\f{x})$.

The correctness of the algorithm then follows from Chernoff bound\cite{boucheron2013concentration} 
\begin{align*}
    \Pr\Big(\left|U-\bb{E}U\right|\ge \frac{\phi_1(a,\sigma)}{4\ell}\Big) 
    &\le 2\exp\Big(-\frac{T\phi_1(a,\sigma)^2}{8\ell^2}\Big) \\
    &\le 2\exp\Big(-\frac{T}{36\pi\ell^2}\Big).
\end{align*}
Moreover, From the definition of $\s{SNR}$, and the fact that $\bb{E}Z^2 = \sigma^2$, we have 
\begin{align*}
 \s{SNR} &\le \frac{1}{\sigma^2} \cdot \max_{\f{x} \in \{0,1\}^n}\max_{i \in [\ell]}\bb{E}\langle \f{\Tr_\gamma(x)}, \f{v}^i \rangle^2  \\
 &\le \frac{1}{\sigma^2} \cdot  \gamma^2 \max_{i \in [\ell]} \left|\left|\f{v}^i\right|\right|_2^2 \\
 &= O(\ell^2\max_{i \in [\ell]}\left|\left|\f{v}^i\right|\right|_2^2/\delta^2) ~\text{for $\gamma= 2\sqrt{2} \ell \sigma/\delta$}.
\end{align*}
\end{proof}

\paragraph{Acknowledgements.} This work is supported in part by NSF awards 2133484, 2127929, and 1934846.

\clearpage

\bibliographystyle{abbrv}

\clearpage

\appendix

%\onecolumn

% {\Large \bf Supplementary Material}

% %\section{Organization}
% In this document, we present all the missing proofs from the main paper. We begin by providing the necessary background on family of sets in Appendix \ref{sec:bacground}. In Appendix~\ref{sec:compute-SuCa} we prove Lemma \ref{lem:compute-SuCa}. Appendix~\ref{sec:Sunion} is dedicated to prove the guarantees of the helper subroutines used in the proof of Lemma~\ref{lem:compute-SuCa}. In Section~\ref{sec:SUnion}, we prove Lemma~\ref{lem:uff} and present its accompanying Algorithm~\ref{algo:SUnion}. The guarantees of Lemma~\ref{lem:uff} follow from Lemma~\ref{lem:ruff} which in turn uses the Lemma~\ref{lem:batchsize} (for \MLC~oracles) and Lemma~\ref{lem:batchsize2} (for \MLR~oracles) to compute $\s{nzcount}$. The proof of Lemma~\ref{lem:ruff} is presented in Section~\ref{sec:Si}, followed by the proofs of Lemma~\ref{lem:batchsize}, and Lemma~\ref{lem:batchsize2} in Section~\ref{sec:mlc-nzcount} and Section~\ref{sec:mlr-nzcount} respectively. 

% Finally, the proof of Theorem~\ref{lem:suff-t} is presented in Appendix~\ref{app:prooflog}.

\section{Jennrich's Algorithm for Unique Canonical Polyadic (CP) Decomposition}\label{sec:jenn}

In this section, we state Jennrich's Algorithm for CP decomposition (see Sec 3.3, \cite{moitra2014algorithmic}) that we use in this paper. Recall that we are provided a symmetric tensor $\ca{A}$ of order $3$ and rank $R$ as input i.e. a tensor $\ca{A}$ that can be expressed in the form below:
\begin{align*}
    \ca{A}=\sum_{r=1}^{R}\underbrace{\f{z}^r\otimes \f{z}^r \otimes  \f{z}^r}.
\end{align*}
Our goal is to uniquely recover the latent vectors $\f{z}^1,\f{z}^2,\dots,\f{z}^R$ from the input tensor $\ca{A}$ provided that the vectors  $\f{z}^1,\f{z}^2,\dots,\f{z}^R$ are linearly independent.  Let $\ca{A}_{\cdot,\cdot,i}$ denote the $i^{\s{th}}$ matrix slice through $\ca{A}$.

\begin{algorithm}[h!]
\caption{\textsc{\textsc{Jennrich's Algorithm}}$(\ca{A})$\label{algo:tensor}}
\begin{algorithmic}[1]
\REQUIRE A symmetric rank-$ R$ tensor $\ca{A} \in \bb{R}^n \otimes \bb{R}^n  \otimes \bb{R}^n$ of order $3$. 

\STATE Choose $\f{a},\f{b}\in \bb{R}^n$ uniformly at random such that it satisfies $\left|\left|\f{a}\right|\right|_2=\left|\left|\f{b}\right|\right|_2=1$.

\STATE Compute  $\f{T}^{(1)} \triangleq \sum_{i \in [n]}\f{a}_i\ca{A}_{\cdot,\cdot,i},\f{T}^{(2)} \triangleq \sum_{i \in [n]}\f{b}_i\ca{A}_{\cdot,\cdot,i}$.

\IF{$\s{rank}(T^{1})<R$}
\STATE Return Error
\ENDIF

\STATE Solve the general eigen-value problem $\f{T}^{(1)}\f{v}=\lambda_v \f{T}^{(2)}\f{v}$.

\STATE Return the  eigen-vectors $\f{v}$ corresponding to the non-zero eigen-values.

\end{algorithmic}
\end{algorithm}

For the sake of completeness, we describe in brief why Algorithm \ref{algo:tensor} works. Note that $\sum_{i \in [n]}\f{a}_i\ca{A}_{\cdot,\cdot,i}$ is the weighted sum of matrix slices through $\ca{A}$ each weighted by $\f{a}_i$. Therefore, it is easy to see that 
\begin{align*}
    &\f{T}^{(1)} \triangleq \sum_{i \in [n]}\f{a}_i\ca{A}_{\cdot,\cdot,i}= \sum_{r=1}^{R}\langle \f{z}^r,\f{a}\rangle \f{z}^r \otimes \f{z}^r = \f{Z}\f{D}^{(1)}\f{Z}^T \\
    &\f{T}^{(2)} \triangleq \sum_{i \in [n]}\f{b}_i\ca{A}_{\cdot,\cdot,i}= \sum_{r=1}^{R}\langle \f{z}^r,\f{b}\rangle \f{z}^r \otimes \f{z}^r= \f{Z}\f{D}^{(2)}\f{Z}^T
\end{align*}
where $\f{Z}$ is a $n \times R$ matrix whose columns form the vectors $\f{z}^1,\f{z}^2,\dots,\f{z}^R$; $\f{D}^{(1)},\f{D}^{(2)}$ are $R \times R$ diagonal matrices whose entry at the $i^{\s{th}}$ position in the diagonal is $\langle \f{z}^r,\f{a}\rangle$ and $\langle \f{z}^r,\f{b}\rangle$ respectively. Clearly, the matrices $\f{T}^{(1)},\f{T}^{(2)}$ are of rank $R$ if and only if the vectors  $\f{z}^1,\f{z}^2,\dots,\f{z}^R$ are linearly independent and therefore, this condition is easy to verify in Steps 3-5. Now if the sufficiency condition is met, then the generalized eigenvalue decomposition will reveal the unknown latent vectors since the eigenvalues are going to be distinct with probability $1$.

\section{Proof of Concept Simulations}

We set $\ell=3$ i.e. we have 3 unknown vectors of dimension $500$. For each of the first two vectors, we design them by randomly choosing 5 indices to be their support along with the constraint that their supports intersect on exactly 2 indices. We choose the third vector so that its support is the union of the supports of the other two unknown vectors. Note that with such a choice of the unknown vectors, the separability assumption in \cite{gandikota2020recovery}  (the support of any unknown vector is not contained within the union of support of the other unknown vectors) no longer holds true. Let $T$ be the number of times each distinct query is repeated to estimate the $\mathsf{nzcount}(\cdot)$ of the query. For each value of $T \in \{5,10,15,20,25,30,35,40,45,50\}$, we simulate our algorithms $100$ times and compute the fraction of times (let’s call this accuracy) the support of the unknown vectors are recovered exactly. In order to recover the support, we run Algorithm \ref{algo:t-iden-supp-rec} with $p=2$ and Algorithm \ref{algo:s-lin-indep-supp-rec} (with $A^{\mathcal{F}}$ just being $\mathcal{A}$) when the support of the unknown vectors are known to be full-rank (hence we can apply Jennrich’s algorithm (Algorithm \ref{algo:tensor}) directly). We can also think of Algorithm \ref{algo:tensor} as a special case of Algorithm \ref{algo:s-indep-supp-rec} for $w=3$ (see Remark \ref{rmk:comp}).
Note that in Algorithm \ref{algo:tensor}, the eigenvectors obtained are not exactly sparse (due to precision issues while solving the generalized eigen-value problem) and has extremely small non-zero values corresponding to the zero entries of the unknown vectors. This can be easily resolved by using a post-processing step on the recovered eigenvectors where we retain only those entries in the support with an absolute value more than $0.002$. Similarly, the zero eigenvalues in Algorithm \ref{algo:tensor} turn out to be small non-zero values in simulation; again, this can be resolved by taking the eigenvectors corresponding to the top $8$ non-zero eigenvalues, modify the corresponding eigenvectors by the aforementioned post-processing step and return the distinct support vectors obtained. In this experiment, the union-free families are simulated by just obtaining a random design which works with high probability. We obtain the following result (here T can be a proxy for the total number of measurements, as the later grows linearly with T):

\begin{center}
\begin{tabular}{||c | c | c||} 
 \hline
 T & Algorithm \ref{algo:t-iden-supp-rec}(Accuracy) & Algorithm \ref{algo:tensor}(Accuracy) \\ [0.5ex] 
 \hline\hline
 5 & 0.04 & 0.0 \\ 
 \hline
 10 & 0.2 & 0.14 \\
 \hline
 15 & 0.33 & 0.19 \\
 \hline
 20 & 0.48 & 0.5 \\
 \hline
 25 & 0.45 & 0.62 \\
 \hline
 30 & 0.72 & 0.8 \\
 \hline
 35 & 0.86 & 0.9 \\
 \hline
 40 & 0.87 & 0.96 \\
 \hline
 45 & 0.89 & 0.99 \\
 \hline
 50 & 0.84 & 0.99 \\  
 \hline
\end{tabular}
\end{center}

It is evident that for both algorithms implemented, the accuracy increases with the number of times a particular vector is repeatedly queried. Comparing the performance of the two algorithms, Jennrich’s algorithm performance improves much faster than Algorithm 1 with the increase of queries. 

\end{document}